\documentclass[journal]{IEEEtran}
\usepackage{amsmath,graphicx}
\usepackage{color}
\usepackage{lineno,hyperref}
\usepackage{amssymb}
\usepackage{multirow}
\usepackage{bm}
\usepackage{booktabs,subfigure}
\usepackage{amsthm}
\usepackage{slashbox}
\usepackage{threeparttable}
\usepackage{makecell}
\usepackage[ruled,linesnumbered]{algorithm2e}
\usepackage{microtype}

\DeclareMathOperator*{\argmin}{argmin}

\DeclareMathOperator*{\e}{e}
\DeclareMathOperator*{\diag}{diag}

\newtheorem{definition}{Definition}

\newtheorem{theorem}{Theorem}
\newtheorem{lemma}{Lemma}
\newtheorem{proposition}{Proposition}

\ifCLASSINFOpdf
\else
\fi

\begin{document}

\title{Indefinite Kernel Logistic Regression with Concave-inexact-convex Procedure}

\author{Fanghui Liu, Xiaolin Huang, Chen Gong, Jie Yang, and Johan A.K. Suykens
\thanks{
This work was supported in part by the National Natural Science Foundation of China under Grant 61572315, Grant 6151101179, Grant 61603248, Grant 61602246, in part by 973 Plan of China under Grant 2015CB856004, in part by the Natural Science Foundation of Jiangsu Province under Grant BK20171430, in part by the Fundamental Research Funds for the Central Universities (No: 30918011319), the open project of State Key Laboratory of Integrated Services Networks (Xidian University, ID: ISN19-03), and the “Summit of the Six Top Talents” Program (No: DZXX-027), in part by FWO G0A4917N, G.0377.12, G.088114N, IUAP P7/19 DYSCO, KU Leuven CoE PFV/10/002 (OPTEC)
(\emph{Corresponding authors: Jie Yang and Xiaolin Huang.})}
\thanks{
 F. Liu, X. Huang and J. Yang are with Institute of Image Processing and Pattern Recognition, Shanghai Jiao Tong University, Shanghai 200240, China (e-mail: lfhsgre@sjtu.edu.cn; xiaolinhuang@sjtu.edu.cn; jieyang@sjtu.edu.cn).
 }
 \thanks{
C. Gong is with the Key Laboratory of Intelligent Perception and Systems for High-Dimensional Information of Ministry of Education, School of Computer Science and Engineering, Nanjing University of Science and Technology, Nanjing, 210094, P.R. China (email: chen.gong@njust.edu.cn).}
\thanks{J.A.K. Suykens is with the Department of Electrical Engineering
(ESAT-STADIUS), KU Leuven, B-3001 Leuven, Belgium (email:
johan.suykens@esat.kuleuven.be).}
}

%\markboth{IEEE TRANSACTIONS ON NEURAL NETWORKS AND LEARNING SYSTEMS, VOL. **, NO. **, ** 2018}%
%{Shell \MakeLowercase{\textit{et al.}}: Bare Demo of IEEEtran.cls for Journals}

\maketitle

\begin{abstract}
In kernel methods, the kernels are often required to be positive definite, which restricts the use of many indefinite kernels.
%To utilize those indefinite kernels, indefinite learning methods are of great interests.
To consider those non-positive definite kernels, in this paper, we aim to build an indefinite kernel learning framework for kernel logistic regression.
The proposed indefinite kernel logistic regression (IKLR) model is analysed in the Reproducing Kernel Kre\u{\i}n Spaces (RKKS) and then becomes non-convex.
Using the positive decomposition of a non-positive definite kernel, the derived IKLR model can be decomposed into the difference of two convex functions.
Accordingly, a concave-convex procedure is introduced to solve the non-convex optimization problem.
Since the concave-convex procedure has to solve a sub-problem in each iteration, we propose a concave-inexact-convex procedure (CCICP) algorithm with an inexact solving scheme to accelerate the solving process.
Besides, we propose a stochastic variant of CCICP to efficiently obtain a proximal solution, which achieves the similar purpose with the inexact solving scheme in CCICP.
The convergence analyses of the above two variants of concave-convex procedure are conducted.
By doing so, our method works effectively not only under a deterministic setting but also under a stochastic setting.
Experimental results on several benchmarks suggest that the proposed IKLR model performs favorably against the standard (positive-definite) kernel logistic regression and other competitive indefinite learning based algorithms.
\end{abstract}


\IEEEpeerreviewmaketitle

%%%%%%%%% BODY TEXT
\section{Introduction}
\label{sec:intro}
\IEEEPARstart{K}{ernel} methods \cite{Sch2003Learning,Vapnik2000The} have been successfully applied to many machine learning tasks such as classification \cite{Zhu2002Kernel,Chen2016Multi}, regression \cite{drucker1997support}, and clustering \cite{dhillon2004kernel}.
In these algorithms, a kernel function $\mathcal{K}(\bm{x}_i,\bm{x}_j)$ is employed to evaluate the similarity between two data points $\bm{x}_i$ and $\bm{x}_j$.
%If the corresponding similarity matrix derived by $\mathcal{K}$ is positive semi-definite (PSD), then it can be served as a kernel matrix $\bm{K}$ in such standard kernel methods.
%That is to say, the kernel matrix $\bm{K}$ associated to a positive definite kernel $\mathcal{K}$ is PSD.
%The corresponding similarity matrix derived by a positive definite (PD) kernel $\mathcal{K}$ can be served as a kernel matrix $\bm{K}$ in such standard kernel methods.
Herein, a positive definite kernel $\mathcal{K}$ results in a positive semi-definite (PSD) kernel matrix $\bm{K}$ to satisfy Mercer's condition.
Consequently, the above approaches with positive definite kernels can be theoretically analyzed in the reproducing kernel Hilbert spaces (RKHS) \cite{Evgeniou2000R}.
%That is, the kernel matrix $\bm{K}$ associated to a positive definite kernel $\mathcal{K}$ is PSD.

Nevertheless, in real-world applications, we might meet some \emph{indefinite} (real, symmetric, but not positive definite) \cite{Schleif2015Indefinite} similarity measures due to the following reasons.
First, one can comprehensively exploit the domain-specific structure in data and accordingly design a certain similarity measure.
The kernel matrix derived by such measure often achieves promising empirical performance without any positive definiteness requirement on it.
%There is no data-driven reason to enforce positive-definiteness in kernel methods.
%For example, in the multimedia community, one can use the human-judged similarities between concepts and words in music recommendation \cite{Wang2014Improving} and video recommendation \cite{Shah2014ADVISOR}.
%In the bioinformatics field,  the Smith Waterman and BLAST score \cite{Saigo2004Protein} is widely used for the protein sequence similarity measures.
%In other fields, one could use the optimal assignment kernels \cite{Kriege2016On} for graph classification, or exploit dynamic time warping \cite{Keogh2005Exact,Marteau2013On} for time series.
For example, one can utilize the Smith Waterman and BLAST score \cite{Saigo2004Protein} for the protein sequence, the optimal assignment kernels \cite{Kriege2016On} for graph classification, or dynamic time warping \cite{Keogh2005Exact,Marteau2013On} for time series.
In these cases, the corresponding kernel matrices generated by such similarities are not PSD.
Second, the developed similarity measurements might be contaminated by outliers or noises \cite{Ying2009Analysis}, that makes the initial PSD kernel matrix degenerate to an indefinite one.
Third, the Mercer condition is difficult to verify even if a kernel is positive definite in essence.
In this case, we have to tackle the kernel as an indefinite one.
Last, a positive definite kernel cannot be guaranteed to be positive definite embed in another space.
For instance, the Gaussian kernel is perhaps the most popular positive definite kernel with widespread applications.
In a Riemannian manifold, the geodesic distance is more accurate than the Euclidean distance \cite{Zhang2015Learning,gong2015deformed}, and accordingly, it seems natural to use the geodesic distance in the Gaussian kernel \cite{Jayasumana2013Kernel}.
%It would seem natural to adapt this kernel to account for the geometry of Riemannian manifolds, that is, replacing the Euclidean distance in the Gaussian kernel with the geodesic distance on the manifold \cite{Jayasumana2013Kernel}.
However, the kernel derived in this manner is not positive definite in general \cite{Feragen2015Geodesic}.
%In above situations, many learning models boil down to be non-convex due to the used indefinite kernel which violates Mercer's condition.
Based on above analyses, there are both algorithmic and theoretical requirements to consider these \emph{indefinite} similarities.
Here we mainly discuss indefinite support vector machine (SVM) in above two aspects.

%To use indefinite similarities in classification task, there have been some discussions, mainly on SVM.
%In theory, due to the non-positiveness of the kernel, Mercer's theorem is no longer valid.
%Accordingly, a proper nonlinear feature mapping should be re-defined from the primal problem to its dual one.
In theory, a proper nonlinear feature mapping for indefinite kernels should be re-defined because Mercer's theorem is no longer valid.
Hence, Ong \emph{et~al.} \cite{Cheng2004Learning} introduce the Reproducing Kernel Kre\u{\i}n Spaces (RKKS) to give a characterization in terms of primal and dual problems for SVM with indefinite kernels \cite{Ga2016Learning}.
Compared with the conventional RKHS for positive definite kernels, the inner products might be negative for RKKS.

The solving algorithms for indefinite kernel learning can be grouped into two categories: spectrum modification and non-convex optimization.
In the first approach, the indefinite kernel matrix is transformed into a PSD one by spectrum modification, and then a regular solver can be used.
For example, ``Flip" \cite{Graepel1999Classification} uses the absolute value of eigenvalues, ``Clip" \cite{Pekalska2002A} sets the nonnegative eigenvalues to zero, ``Shift" \cite{roth2003optimal} plus a positive constant to all eigenvalues until the smallest one is zero.
However, the above operations actually change the indefinite matrix itself,  which results in the inconsistency between the training and the test kernel. %since the above operators can only be used for training data.
%As a result, such operation may cause the loss of some important information involved with the kernel.
Comparably, some solvers can directly deal with the non-convex dual problem within indefinite kernel matrices.
In \cite{CC01a}, SVM with indefinite kernels can be still solved by the SMO-type algorithm.
Note that this algorithm still converges but the solution is just a stationary point.
Besides, in \cite{Akoa2008Combining,xu2017solving}, the authors directly introduce a non-convex approach, the concave-convex procedure (CCCP) \cite{Yuille2003The} to solve such problem effectively.
%N

%However, the CCCP algorithm must tackle a sub-problem in each iteration, which makes the solving process very inefficient.
%Consequently, Artacho $et~al.$ \cite{Artacho2015Accelerating} accelerate the CCCP with a line search using an Armijo type rule with a large reduction.
%In the last category, learning in the Reproducing Kernel Kre\u{\i}n Spaces (RKKS) \cite{Cheng2004Learning} provides an adequate justification for feature space interpretation in theory.
%Therein, the similarity kernel is decomposed into the sum of one positive definite kernel and one negative definite kernel.
%By such decomposition, learning in Kre\u{\i}n spaces sufficiently exploits the intrinsic information carried out by the negative part in the indefinite kernel.
%In \cite{Ga2016Learning}, the dual form of SVM with indefinite kernels is formulated as a stabilization problem instead of a minimization problem, which is well analyzed with solid theoretical foundations.
%However, these approaches must suffer from a dual gap between the primal and the dual problems, which brings about unclear interpretation in the primal space.

%Based on above observations, SVM (including its variants) with indefinite kernels has been thoroughly analysed and successfully applied in various scenarios.
In this paper, we focus on kernel logistic regression (KLR) with indefinite kernels.
KLR is a powerful and representative classifier and has been shown to be effective for classification tasks.
However, KLR equipped with indefinite kernels has not yet been investigated in the past.
Formally,
%Hence, we need to carefully discuss the indefinite model and its corresponding optimization algorithm.
%In formulation, based on the representer theorem in RKKS, we directly focus on the non-convex primal form of the IKLR model, which shares a similar formulation as that of the regular KLR.
%However, using indefinite kernels makes the optimization problem non-convex and hard to solve. To tackle this issue, we decompose the objective function into the difference of two convex functions and then the CCCP algorithm is applicable.
%Moreover, aiming at larger problems in practice, a concave-inexact-convex procedure (CCICP) algorithm is proposed to obtain early stop during each iteration.
%We theoretically demonstrate the convergence of CCICP with the provable guarantee of the upper bounded error.
%Further, it is natural to incorporate stochastic gradient descent (SGD) into CCCP, which achieves the similar effect with the inexact operation in the CCICP algorithm.
%Hence, the proposed method works effectively not only under a deterministic setting but also under a stochastic setting with theoretical convergence analyses.
%Experiments on various multi-modal data sets suggest that in most cases our IKLR method outperforms not only the conventional KLR with positive kernels but also other recent algorithms with indefinite kernels.
the contributions of this paper are summarized as follows.
\begin{enumerate}
  \item We build the IKLR model in the Reproducing Kernel Kre\u{\i}n Spaces (RKKS), and directly focus on its non-convex primal form, the formulation of which keeps consistency to the standard KLR.
  \item To accelerate the solving process, we propose a concave-inexact-convex procedure (CCICP) algorithm with an early termination technique to obtain a proximal solution during each iteration.
      We provide an convergence analysis of such approximation algorithm.
  \item A stochastic variant of CCICP is proposed with convergence guarantees, which achieves the similar effect with the inexact solving scheme.
\end{enumerate}

%Experiments on various multi-modal data sets suggest that in most cases our IKLR model outperforms not only the conventional KLR with positive kernels but also other recent algorithms with indefinite kernels.

This paper is the extended version of our previous work \cite{SgrE_IKLR}.
Apart from details added in several sections, the main extension contains three parts:
First, we incorporate stochastic gradient descent (SGD) into CCCP to solve the proposed IKLR model, and then a stochastic variant of CCICP is presented to further accelerate the solving process.
Second, the convergence analysis of such stochastic optimization algorithm is theoretically demonstrated.
Third, we provide more experiments results on several benchmarks.
And accordingly, more parameters comparisons analysis and computational complexity analysis are also provided.

The remainder of the paper is organized as follows. Section \ref{sec:KLR} briefly reviews kernel logistic regression.
Section \ref{sec:IKLR} introduces the proposed IKLR model.
The optimization algorithm of the IKLR model is presented in Section \ref{sec:CCICP}.
Experimental results on several datasets are presented in Section \ref{sec:experiment}, and conclusion is given in Section \ref{sec:conclusion}.

\section{Review: Kernel Logistic Regression}
\label{sec:KLR}

In this section, we briefly review the regular kernel logistic regression in the binary classification task.
Let $\big\{ (\bm{x}_i,y_i) \big\}^{n}_{i=1}$ with its label $y_i \in \{+1,-1\}$ be $n$ training points, we concern the inference of a function $f: \mathcal{X} \rightarrow \mathcal{Y}$ that predicts a target $y \in \mathcal{Y}$ of a data point $\bm x \in \mathcal{X}$.
%We aim to learn a function $f: \mathcal{X} \rightarrow \mathcal{Y}$ based on these $n$ training data, so that when given a new input $\bm{z} \in \mathbb{R}^m$ ($m$ is the feature dimension) from the test sample set $\bm{Z}=[\bm{z}_1,\bm{z}_2,\cdots,\bm{z}_s]$ with $s$ test samples, its label $y$ can be predicted.
%Many people have noted the relationship between a classifier (e.g. SVM, logistic regression) and regularized function estimation in RKHS \cite{Evgeniou2000R}.
%For instance, fitting a logistic regression problem is equivalent to:
KLR can be fit in the regularization framework of \emph{loss}+\emph{penalty} using the exponential loss function $\ell(f) = \ln (1+e^{-yf})$.
And thus, for a given positive regularization parameter $\lambda$, KLR is the minimum of the following regularized empirical risk functional
\begin{equation}\label{logit}
\begin{split}
\mathop{\mathrm{min}}\limits_{f\in \mathcal{H}}~~ \frac{1}{n}\sum_{i=1}^n\ln \Big(1+e^{-y_if(\bm{x}_i)} \Big) + \frac{\lambda}{2} \|f\|^2_{\mathcal{H}} \,,
\end{split}
\end{equation}
with the RKHS $\mathcal{H}$ generated by the positive definite kernel $\mathcal{K(\cdot,\cdot)}$.
%Generally, the discriminant function is formulated as $f(\bm{x})=\bm{\alpha}^{\top}\bm{x}+b$ \footnote{We omit the bias term in theoretical discussions for simplicity but include it in numerical experiments.}, where $\bm{\alpha} \in \mathbb{R}^m$ is a weight vector parameterizing the space of linear functions mapping from $\mathcal{X}$ to $\mathcal{Y}$.
Using the representer theorem \cite{Lkopf2000A} in RKHS, the optimal function $f^*$ is
\begin{align*}\label{fex}
  f^*=\sum_{i=1}^n\alpha_i\mathcal{K}(\bm{x}_i,\cdot)\,,
\end{align*}
where ${\bm \alpha}=[\alpha_1,\alpha_2,\cdots,\alpha_n]^{\top}$ is the coefficient vector.
Combining this to Eq.~\eqref{logit}, KLR is reformulated as
\begin{equation}\label{klr}
  \mathop{\mathrm{min}}\limits_{\bm \alpha}
  \frac{1}{n}\sum_{i=1}^n\ln \Big(1+\exp\big(-y_i\sum_{j=1}^n\alpha_j\bm{K}_{ij}\big)\Big) +  \frac{\lambda}{2}\sum_{i,j=1}^n\alpha_i\alpha_j\bm{K}_{ij} \,,
\end{equation}
with $\bm{K}_{ij}=\mathcal{K}(\bm{x}_i,\bm{x}_j)$.
After some straightforward algebraic manipulations, we obtain a compact form of KLR
\begin{equation}\label{klrmat}
  \mathop{\mathrm{min}}\limits_{{\bm \alpha}}~~
  \frac{1}{n}\bm{1}^{\!\top}\!\ln \Big(\bm{1}+\exp(-\bm{y} \odot \bm{K}{\bm \alpha})\Big) + \frac{\lambda}{2}\bm{{\bm \alpha}}^{\top}\bm{K}{\bm \alpha} \,,
\end{equation}
where $\bm 1$ is the all-one vector, $\bm y$ is the label vector, and $\odot$ denotes the Hadamard product.
Traditionally, $\bm{K}$ in Eq.~\eqref{klrmat} is required to be positive semi-definite, and accordingly the optimization problem is formulated as a convex unconstrained quadratic programming.
%Thereby, the Newton-Raphson method can be used to iteratively solve the objective function.

\section{Indefinite Kernel Logistic Regression Model}
\label{sec:IKLR}
The functional space spanned by indefinite kernels belong to the Reproducing Kernel Kre\u{\i}n Spaces (RKKS) \cite{bognar1974indefinite} instead of RKHS.
We first introduce Kre\u{\i}n spaces and then derive the IKLR model.

%Kre\u{\i}n spaces are indefinite inner product spaces endowed with a Hilbertian topology, which are able to generate negative inner products.

\begin{definition}\label{definiterkks}
(Kre\u{\i}n space \cite{bognar1974indefinite}) An inner product space is a Kre\u{\i}n space $\mathcal{H_K}$ if there exist two Hilbert spaces $\mathcal{H}_+$ and $\mathcal{H}_-$ such that
1) All $f \in \mathcal{H_K}$ can be decomposed into $f=f_++f_-$, where $f_+ \in \mathcal{H}_+$ and $f_- \in \mathcal{H}_-$, respectively.
2) $\forall f,g \in \mathcal{H_K}$, $\langle f,g\rangle_{\mathcal{H_K}}=\langle f_+,g_+\rangle_{\mathcal{H}_+} - \langle f_-,g_-\rangle_{\mathcal{H}_-}$.
\end{definition}
If $\mathcal{H}_+$ and $\mathcal{H}_-$ are RKHSs, $\mathcal{H_K}$ is a RKKS with a unique indefinite kernel $\mathcal{K}$ such that the reproducing property holds: for all $f \in \mathcal{H_K},~f(\bm x) = \langle f,k(\bm x,\cdot)\rangle_{\mathcal{H_K}}$.
In this space, the squared norm and the squared distance\footnote{A corresponding squared distance is defined as $d^2(x,x')=\mathcal{K}(x,x)- 2\mathcal{K}(x,x')+\mathcal{K}(x',x')$.} induced by an indefinite kernel $\mathcal{K}$ can be negative in contrast to the Euclidean case.
This definition may not define a metric, as it violates the triangle inequality.
However, this squared distance function is able to provide a justification of data representation in this vector space.
Details about the interpretation of SVM with indefinite kernels in the feature space can be found in \cite{Haasdonk2005Feature}.

Based on above analyses, our IKLR model with an indefinite kernel $\mathcal{K}$ is formulated as
\begin{equation}\label{logitkre}
\begin{split}
\mathop{\mathrm{min}}\limits_{f\in \mathcal{H_K}} ~ \frac{1}{n}\sum_{i=1}^n \ln \Big(1+e^{-y_if(\bm{x}_i)} \Big) + \frac{\lambda}{2} \|f\|^2_{\mathcal{H_K}}\,,
\end{split}
\end{equation}
with the RKKS $\mathcal{H_K}$ generated by the indefinite kernel $\mathcal{K(\cdot,\cdot)}$.
By the representer theorem in RKKS \cite{Cheng2004Learning}, the optimal $f^*$ admits
%That is, if the optimization problem in Eq.~\eqref{logitkre} has a saddle point, it admits the following expansion:
 \begin{align*}\label{fexx}
  f^*=\sum_{i=1}^n\alpha_i\mathcal{K}(\bm{x}_i,\cdot)\,,
\end{align*}
Accordingly, Eq.~\eqref{logitkre} can be rewritten as
\begin{equation}\label{iklrmat}
  \mathop{\mathrm{min}}\limits_{{\bm \alpha}}~
  \frac{1}{n}\bm{1}^{\top}\ln \Big(\bm{1}+\exp(-\bm{y} \odot \bm{K}{\bm \alpha})\Big) + \frac{\lambda}{2}\bm{{\bm \alpha}}^{\top}\bm{K}{\bm \alpha}\,.
\end{equation}
%where the label matrix $\bm{Y} \in \mathbb{R}^{n\times n}$ is a diagonal matrix, of which the $i$th diagonal element is $y_i$.
One can see that the proposed IKLR model in Eq.~\eqref{iklrmat} is similar to the regular KLR in Eq.~\eqref{klrmat}, but it must be analysed in RKKS and becomes non-convex because of the indefinite kernel matrix.
%And also, the optimization problem is our IKLR model becomes non-convex and can be solved by CCCP.

To solve such non-convex problem, we also need the following proposition.
\begin{proposition}
(\cite{Cheng2004Learning})
A non-positive definite kernel $\mathcal{K}$ in RKKS admits a positive decomposition on a given set
 \begin{equation*}
    \mathcal{K}(\bm{x}_i,\bm{x}_j) = \mathcal{K}_+(\bm{x}_i,\bm{x}_j) - \mathcal{K}_-(\bm{x}_i,\bm{x}_j), \forall \bm{x}_i, \bm{x}_j\in \mathcal{X}\,,
 \end{equation*}
 with two positive definite kernels $\mathcal{K}_+$ and $\mathcal{K}_-$.
\end{proposition}
Hence, the proposed IKLR model in Eq.~\eqref{iklrmat} can be further expressed as
\begin{equation}\label{iklrmain}
  \mathop{\mathrm{min}}\limits_{{\bm \alpha}}~
  \frac{1}{n}\bm{1}^{\!\top}\!\ln \Big(\bm{1}+\exp(-\bm{y}\odot \bm{K}{\bm \alpha})\Big) + \frac{\lambda}{2}\bm{{\bm \alpha}}^{\top}(\bm{K}_+\!-\bm{K}_-){\bm \alpha} \,,
\end{equation}
with two PSD kernel matrices $\bm{K}_+$ and $\bm{K}_-$, which can be obtained by eigenvalue decomposition of $\bm{K}$.
To be specific, $\bm{K}=\bm{V}^{\top}{\bm \Lambda} \bm{V}$,  where $\bm{V}$ is an orthogonal matrix and the diagonal matrix is ${\bm \Lambda} = \diag (\mu_1,\mu_2,\cdots,\mu_n)$  with eigenvalues $\mu_1 \geq \mu_2 \geq \cdots \geq\mu_n$.
 Without loss of generality, suppose that the first $s$ eigenvalues are nonnegative and the remaining $n-s$ ones are negative, $\bm{K}_+$ and $\bm{K}_-$ can thus be given by
\begin{align*}\label{Kdef}
\left\{
\begin{array}{rcl}
\begin{split}
&\bm{K}_+ = \bm{V}^{\top}\diag(\mu_1+\tau,\dots,\mu_v+\tau,\tau,\dots,\tau)\bm{V} \\
&\bm{K}_- = \bm{V}^{\top}\diag(\tau,\dots,\tau, \rho-\mu_{v+1},\dots,\tau-\mu_n)\bm{V}\,,
\end{split}
\end{array} \right.
\end{align*}
where $\tau$ is chosen by $\tau>-\mu_n$ to ensure that these two matrices $\bm{K}_+$ and $\bm{K}_-$ are positive semi-definite.
After conducting this positive decomposition, we decompose the objective function in Eq.~\eqref{iklrmain} as difference of two convex functions $g({\bm \alpha})$ and $h({\bm \alpha})$
\begin{equation}\label{maindef}
\left\{
\begin{array}{rcl}
\begin{split}
&g({\bm \alpha})=
  \frac{1}{n}\bm{1}^{\top}\ln \Big(\bm{1}+\exp(-\bm{y}\odot \bm{K}{\bm \alpha})\Big) + \frac{\lambda}{2}\bm{{\bm \alpha}}^{\top}\bm{K}_+{\bm \alpha}\\
&h({\bm \alpha})=\frac{\lambda}{2}\bm{{\bm \alpha}}^{\top}\bm{K}_-{\bm \alpha} \,.
\end{split}
\end{array} \right.
\end{equation}

\section{CCICP Optimization for IKLR}
\label{sec:CCICP}
This section first introduces a CCICP algorithm with two approximation schemes to solve the non-convex optimization problem, and then provides convergence analyses of the proposed CCICP algorithm.

\subsection{CCICP in the IKLR Model}
The concave-convex procedure \cite{Yuille2003The} is a typical non-convex algorithm to solve d.c. (difference of convex functions) programs.
By decomposing the non-convex objective function in Eq.~\eqref{iklrmain} into difference of two convex functions $g({\bm \alpha})$ and $h({\bm \alpha})$, the CCCP is an iterative procedure with
\begin{equation*}
  \bm \alpha_{k+1} \in \argmin_{\bm \alpha} ~g({\bm \alpha}) - \bm \alpha^{\!\top} \nabla h({\bm \alpha}_{k})\,.
\end{equation*}
The core idea of CCCP is to linearize the concave part of the non-convex objective function, \emph{i.e.}, $-h({\bm \alpha})$, around its current solution ${\bm \alpha}_k$.
At each iteration, its convex approximation is formulated as
\begin{equation}\label{faak}
\mathcal{F}_k({\bm \alpha}) \triangleq \mathcal{F}({\bm \alpha},{\bm \alpha}_k) = g({\bm \alpha})  -\big[h({\bm \alpha}_k) + \nabla h^{\top}({\bm \alpha}_k)({\bm \alpha}-{\bm \alpha}_k)\big]\,.
\end{equation}
Here $\mathcal{F}_k({\bm \alpha})$ can be solved by an off-the-shelf convex algorithm such as the gradient descent method to obtain ${\bm \alpha}_{k+1}$.
%Theoretical analyses suggest that CCCP is able to converge to a stationary point \cite{Sriperumbudur2009On}.
To be specific, in our model, $h({\bm \alpha})$ is replaced by its first order Taylor approximation at ${\bm \alpha}_{k}$
\begin{equation*}
  \tilde{h}({\bm \alpha}_{k})=h({\bm \alpha}_{k})+\lambda{\bm \alpha}_k^{\top}\bm{K}_-({\bm \alpha}-{\bm \alpha}_{k})\,.
\end{equation*}
Combining this to Eq.~\eqref{faak}, we have
\begin{equation}\label{appobj}
\mathcal{F}_k({\bm \alpha}) = \frac{\lambda}{2}\bm{{\bm \alpha}}^{\top}\bm{K}_+{\bm \alpha}
  +\frac{1}{n}\bm{1}^{\!\!\top}\!\ln \! \Big(\bm{1}+\exp(-\bm{y}\odot\bm{K}{\bm \alpha})\Big) - \tilde{h}({\bm \alpha}_{k})\,.
\end{equation}

Nonetheless, one can see that, at each iteration, the CCCP needs to solve the sub-problem, which makes CCCP inefficient especially for a large scale problem.
Based on this, we attempt to obtain an inexact solution of the sub-problem to speed up the solving process, termed as the concave-inexact-convex procedure (CCICP).
To this end, we develop two approaches, one is the gradient descent method with an early termination condition, termed ``CCICP-GD", and the other is incorporated with SGD to achieve the similar effect, termed ``CCICP-SGD".
%By doing so, the CCICP algorithm is able to effectively speed up the solving process under a deterministic or stochastic setting.

\subsubsection{Solving with CCICP-GD}
In our CCICP-GD method, the sub-problem is solved by the gradient descent method, in which the gradient $\nabla_{{\bm \alpha}}\mathcal{F}_k({\bm \alpha})$  is
\begin{equation}\label{grad}
  \nabla_{{\bm \alpha}}\mathcal{F}_k({\bm \alpha})=\lambda \bm{K}_+{\bm \alpha} -\frac{1}{n} \bm y \odot \bm{K}\bm{\beta}-\lambda\bm{K}_-{\bm \alpha}_{k}\,,
\end{equation}
where $\bm{\beta} = ({\beta_1},{\beta_2},\dots,{\beta_n})^{\top}$ is defined by
\begin{equation}\label{qi}
  \beta_i = \frac{1}{1+\exp\big(y_i\sum_{j=1}^n\beta_j\bm{K}_{ij}\big)}, \quad \forall i=1,2,\dots,n.
\end{equation}
Using an early termination condition, the inexact solution ${\bm \alpha}_{k+1} \triangleq {\bm \alpha}_{k}^{(T)}$ after $T$ iterations is obtained by ${\bm \alpha}_{k+1}\approx\mathop{\mathrm{argmin}}\limits_{\bm \alpha} \mathcal{F}_k({\bm \alpha})$.
%To obtain the inexact solution ${\bm \alpha}_{k+1}\approx\mathop{\mathrm{argmin}}\limits_{\bm \alpha} \mathcal{F}_k({\bm \alpha})$ in the sub-problem, the terminate condition occupied by Eq.~\eqref{gradbe} in gradient descent is executed to obtain early stop during each iteration.
%Specifically, under the inexact solving scheme with the bounded error assumption, the rationality of such approximation and the convergence of the CCICP algorithm will be theoretically demonstrated in Section \ref{sec:AC}.
In Section \ref{sec:AC}, we detail the definition of such early termination condition and then theoretically demonstrate the convergence analyses of such approximation.

\subsubsection{Solving with CCICP-SGD}
The stochastic gradient descent method can also achieve the similar effect with such inexact scheme to solve the sub-problem.
Since it only processes a mini-batch of data points \cite{Cotter2011Better} or even one data point \cite{Mitliagkas2013Memory} in each iteration, the computational cost per iteration dramatically decreases.
To be specific, in our algorithm, we randomly pick up only one data point to compute the gradient in the sub-problem.
The updating scheme is given by
\begin{equation}\label{SGD}
{\bm \alpha}^{(t+1)} \leftarrow {\bm \alpha}^{(t)} - \eta_t \big( -y_j\bm{K}_j\beta_j + \lambda \bm{K}_+{\bm \alpha}^{(t)} \big)\,,
\end{equation}
where $\eta_t$ is the step size in the $t$th iteration, and $\bm{K}_j$ represents the (randomly picked) $j$th column of the kernel matrix $\bm{K}$.
Finally, the detailed procedure of the CCICP algorithm with GD and SGD for IKLR is summarized in Algorithm \ref{alg:one1}.

\begin{algorithm}
\SetAlgoNoLine
\KwIn{an indefinite kernel matrix $\bm{K}$ and its positive decomposition $\bm{K}_+$ and $\bm{K}_-$.}
\KwOut{the coefficient vector ${\bm \alpha}$.}
Set:  stopping criterion: the inexact parameter $\epsilon=1$, $k_{\max}=20$, the learning rate $\eta=0.02$, and $\rho=0.8$\;
Initialize $k=0$ and ${\bm \alpha}_0$\;
The parameter $\lambda$ is chosen by cross-validation\;
\SetKwRepeat{RepeatUntil}{Repeat}{Until}
\RepeatUntil{$k=k_{\max}$ }
{Obtain $\tilde{h}({\bm \alpha}_k)=\lambda\bm{K}_-{\bm \alpha}_k$ and the sub-problem $\mathcal{F}_k({\bm \alpha})$ by Eq.~\eqref{appobj}\;
\tcp*[h]{Solve the sub-problem.}\\
Initialize $t=0$ and compute $\mathcal{F}_k({\bm \alpha}^{(0)}_k)$ \;
\While{$\| \mathcal{F}_k({\bm \alpha}^{(t+1)}_k) - \mathcal{F}_k({\bm \alpha}^{(t)}_k)\| > \epsilon $ }{
Obtain the gradient $\nabla \mathcal{F}_k({\bm \alpha}^{(t)}_k)$ by Eq.~\eqref{grad}\;
GD: ${\bm \alpha}^{(t+1)}_{k}:={\bm \alpha}^{(t)}_{k}-\eta_t\nabla \mathcal{F}_k({\bm \alpha}^{(t)}_k)$\;
SGD: Randomly pick a $j\in \{ 1,2,\cdots,n\}$ and then update ${\bm \alpha}^{(t+1)}_{k}$ by Eq.~\eqref{SGD}\;
$\eta := \rho \eta$ \;
$t := t + 1$\;

}
Output the inexact solution ${\bm \alpha}_{k+1} := {\bm \alpha}_k^{(t)}$ of Eq.~\eqref{appobj}\;
\tcp*[h]{Complete the inner loop.}\\
$k := k + 1$\;
}
Output the stationary point ${\bm \alpha}_{k_{\max}}$ of Eq.~\eqref{maindef}.
\caption{CCICP for the IKLR model.}
\label{alg:one1}
\end{algorithm}

One can see that Algorithm \ref{alg:one1} contains two loops.
In each iteration of the outer loop, an unconstrained quadratic programming is solved by gradient descent and its complexity is $\mathcal{O}(dn)$, where $d$ is the feature dimension and $n$ is the number of training examples.
Finally, the total computational complexity of our CCICP algorithm is $\mathcal{O}(Tkdn)$, where $T$ is the number of convergence iterations and $k$ is the number of classes.

When we obtain $\bm \alpha^*$ by Algorithm~\ref{alg:one1}, a test data point $\bm z$ can be predicted by
\begin{align*}\label{score}
  p(\bm{z})=\frac{\exp\big(\bm{K_z}{\bm \alpha^*} \big)} {1+\exp\big(\bm{K_z}{\bm \alpha^*}\big)}\,,
\end{align*}
with $\bm{K_z}=[\mathcal{K}(\bm{x}_1,\bm{z}), \mathcal{K}(\bm{x}_2,\bm{z}), \dots \mathcal{K}(\bm{x}_n,\bm{z})]$.
If $p(\bm z) \geq 0.5$, its label is predicted by $+1$, and $-1$ otherwise.

\subsection{Analysis of CCICP-GD}
\label{sec:AC}
This section investigates the convergence of the proposed CCICP-GD.
Since the gradient descent algorithm is used to solve the sub-problem, it satisfies $\mathcal{F}_{k}({\bm \alpha}_{k+1})\triangleq \mathcal{F}_k({\bm \alpha}_k^{(T)})\leq \mathcal{F}_k({\bm \alpha}_k^*)$.
%Herein, the inexact solution ${\bm \alpha}_k^{(T)}$ lies in a $\delta({\bm \alpha})$-neighbor of the optimal result ${\bm \alpha}^*_{k}=\mathop{\mathrm{argmin}}\limits_{\bm \alpha} \mathcal{F}_k({\bm \alpha})$:
Herein, the inexact solution ${\bm \alpha}_k^{(T)}$ satisfies
 %Here ${\bm \alpha}^{(t+1)}$ is bounded by ${\bm \alpha}^{(t)}_*$ with the following formula:
 \begin{align*}
   {\bm \alpha}^{(T)}_k \in \text{U}_{\delta({\bm \alpha})}({\bm \alpha}_k^{*})\triangleq \big\{ {\bm \alpha} \mid \| {\bm \alpha} - {\bm \alpha}_k^* \| \leq \delta({\bm \alpha}) \big\}\,,
 \end{align*}
 where ${\bm \alpha}^*_{k}=\mathop{\mathrm{argmin}}\limits_{\bm \alpha} \mathcal{F}_k({\bm \alpha})$ is the optimal result.
  The notation $\delta({\bm \alpha})$ depends on the current solution, and it should be bounded to guarantee the convergence of CCICP.
 In this case, such approximation solution ${\bm \alpha}_k^{(T)}$ does not satisfy the Karush-Kuhn-Tucker (KKT) condition.
Suppose that
\begin{equation}\label{gradbe}
\nabla_{{\bm \alpha}} \mathcal{F}_k({\bm \alpha}) |_{{\bm \alpha}={\bm \alpha}_k^{(T)}}  = \epsilon \| {\bm \alpha}_{k}\| \neq 0\,,
\end{equation}
where $\epsilon$ depends on $\delta({\bm \alpha})$, and its choice will be analysed to guarantee the convergence of CCICP-GD in the following description.

%With the aforementioned inexact operation, the CCICP algorithm is expected to speed up the optimization process.
%For the ease of such algorithm in theory, we carefully consider the convergence of CCICP by investigating an inexact sequence
%$\{{\bm \alpha}_k\}_{k=0}^{\infty}$ generated by Algorithm \ref{alg:one1}, and then further analyse its convergence rate in the proposed IKLR model.

%Based on this, we consider the convergence of the iterative procedure like CCCP with bounded errors in our IKLR model.
%With the aforementioned definition (i.e. the approximated KKT condition),
The main result for CCICP-GD is demonstrated by Theorem \ref{theo}, that is, when $\epsilon$ is upper bounded, the sequence $\{{\bm \alpha}_k \}_{k=0}^{\infty}$ generated by CCICP with an initial point ${\bm \alpha}_0\in \mathbb{R}^n$ still converges.
Before we proceed with the proof of Theorem \ref{theo}, we need the following Lemma \ref{lemmar}.
\begin{lemma}\label{lemmar}
Given a sigmoid function $R(x)=({1+\e^{cx}})^{-1}$ with $c\in\{+1,-1\}$ on $\mathbb{R}$, for any $ x_1 < x_2$, we have
\begin{equation}\label{boundr}
  \big| R(x_1) - R(x_2) \big| \leq \frac{1}{4} \big| x_1 - x_2 \big|\,.
\end{equation}
\end{lemma}
\begin{proof}
Since $R(x)$ is a differentiable function, by the Lagrange mean value theorem, there exists at least one point $\xi \in \big(x_1,x_2\big)$ such that
\begin{align*}\label{Lmean}
  \big| R(x_1) - R(x_2) \big| = \big| (x_1 - x_2)R'(\xi)\big|\,,
\end{align*}
where $|R'(\xi)|$ admits
\begin{align*}
  |R'(\xi)| = \frac{\e^{c\xi}}{(1+\e^{c\xi})^2} = \frac{1}{\e^{c\xi}+\e^{-c\xi}+2}\leq \frac{1}{4}\,,
\end{align*}
which concludes the proof.
\end{proof}
We now present the convergence theorem for CCICP-GD.
\begin{theorem}\label{theo}
Let $\{{\bm \alpha}_k \}_{k=0}^{\infty}$ be any sequence generated by CCICP-GD, its limit point is a stationary point if $\epsilon$ in Eq.~\eqref{gradbe} satisfies
\begin{equation}\label{vare}
  \epsilon < \lambda (\|\bm{K}_+\|- \| \bm{K}_-\|) - \frac{\|\bm{K}\|^2}{4n}\,.
\end{equation}
\end{theorem}
\begin{proof}
Let $\phi: U \subset \mathbb{R}^n \rightarrow \mathbb{R}^n$ be a point-to-set map such that
\begin{align*}
  \phi({\bm \alpha}_{k}) = \mathop{\mathrm{argmin}} \limits_{\bm \alpha}\mathcal{F}_k({\bm \alpha})\,,
\end{align*}
which generates an inexact sequence $\{{\bm \alpha}_k \}_{k=0}^{\infty}$.
Besides, $\phi({\bm \alpha}_k)$ satisfies
\begin{align*}
\nabla_{{\bm \alpha}} \mathcal{F}_k({\bm \alpha}) |_{{\bm \alpha}=\phi({\bm \alpha}_k)}  = \epsilon \| {\bm \alpha}_{k}\| \,.
\end{align*}
%Specifically, the map $\phi$ is said to be  \emph{global convergent}\footnote{It does not imply convergence to a global optimum for all initial values ${\bm \alpha}_{0}$.} if for any chosen initial point ${\bm \alpha}_0$, the sequence converges to a point for which a necessary condition of optimality holds.
In the next, we aim to prove that the map $\phi$ is a non-expansive mapping for two arbitrary points $\bm{p}, \bm{q} \in int(U)$ such that
\begin{align*}
  \big\| \phi(\bm{p}) - \phi(\bm{q})\big\| \leq \kappa \| \bm{p} - \bm{q} \|\,,
\end{align*}
where the non-expansive coefficient is $\kappa \in [0,1)$.
Suppose that $ \phi(\bm{p})$ and $ \phi(\bm{q})$ satisfy
\begin{eqnarray}
  \nabla_{{\bm \alpha}} \mathcal{F}({\bm \alpha}, \bm{p}) |_{{\bm \alpha}=\phi(\bm{p})} &=& \epsilon_1 \|\bm{p} \|\,, \label{grada}\\
  \nabla_{{\bm \alpha}} \mathcal{F}({\bm \alpha}, \bm{q}) |_{{\bm \alpha}=\phi(\bm{q})}  &=& \epsilon_2 \|\bm{q} \|\,, \label{gradb}
\end{eqnarray}
%\begin{equation}\label{grada}
%\nabla_{{\bm \alpha}} \mathcal{F}({\bm \alpha}, \bm{a}) |_{{\bm \alpha}=\phi(\bm{a})}  = \epsilon_1 \|\bm{a} \|\,,
%\end{equation}
%\begin{equation}\label{gradb}
%\nabla_{{\bm \alpha}} \mathcal{F}({\bm \alpha}, \bm{b}) |_{{\bm \alpha}=\phi(\bm{b})}  = \epsilon_2 \|\bm{b} \|\,,
%\end{equation}
where $\epsilon_1$ and $\epsilon_2$ correspond to the bounded error.
For simplicity, suppose $\epsilon_1 \leq \epsilon_2$, so the difference between Eqs.~\eqref{grada} and \eqref{gradb} is given by
%\footnote{If $\epsilon_1 > \epsilon_2$, we use the subtraction between Eq.~\eqref{gradb} and \eqref{grada}.}:
\begin{equation}\label{substr}
\lambda\bm{K}_+\big[\phi(\bm{p})-\phi(\bm{q})\big] = \lambda \bm{K}_-(\bm{p}-\bm{q})+\frac{1}{n}\bm y \odot \bm{Kh} + \epsilon_1 \|\bm{p} \| - \epsilon_2 \|\bm{q} \|\,,
\end{equation}
where $\bm{h}=[h_1,h_2,\dots,h_n]^{\!\top}$ is defined by
\begin{align*}
  h_i = \frac{1}{1+\exp\big(y_i\bm{K}^{(i)}\phi(\bm{p})\big)} - \frac{1}{1+\exp\big(y_i\bm{K}^{(i)}\phi(\bm{q})\big)}\,.
\end{align*}
Using Lemma \ref{lemmar}, we have
\begin{align*}
|h_i| \leq \frac{1}{4} | \bm{K}^{(i)} \phi(\bm{p}) - \bm{K}^{(i)} \phi(\bm{q})|, \quad \forall i=1,2,\dots,n\,,
\end{align*}
and accordingly $\|\bm{h}\|_{\infty}$ satisfies\footnote{Here we use $|\bm{p}^{\top}\bm{q}|=\|\bm{p}\|_a\|\bm{q}\|_b$, where $\frac{1}{a}+\frac{1}{b}=1$.}
\begin{small}
\begin{align*}
\begin{split}
  \| h \|_{\infty}  & \leq  \frac{ \big| \bm{K}^{(s)} \phi(\bm{p}) - \bm{K}^{(s)} \phi(\bm{q}) \big|}{4} \leq \frac{\|\bm{K}^{(s)}\|_1\!\cdot\!\|\phi(\bm{p})\!-\!\phi(\bm{q})\|_{\infty}}{4} \\
  &\leq \frac{1}{4} \|\bm{K}\|_{\infty}\|\phi(\bm{p})-\phi(\bm{q})\|_{\infty}\,,
  \end{split}
\end{align*}
\end{small}
with $s\!=\!\mathop{\mathrm{argmin}} \limits_{i} \big| \bm{K}^{(i)} \phi(\bm{p}) - \bm{K}^{(i)} \phi(\bm{q})\big|, i\!=\!1,2,\dots,n$.
Since $\bm{K}_+$ is positive definite, Eq.~\eqref{substr} can be rewritten as
\begin{equation*}
\begin{split}
  & \phi(\bm{p}) - \phi(\bm{q}) \\
  & = \frac{1}{\lambda}\bm{K}_+^{-1} \Big\{\lambda \bm{K}_-(\bm{p}-\bm{q})+\frac{1}{n} {\bm y \odot \bm{Kh}} + \epsilon_1 \|\bm{p} \| - \epsilon_2 \|\bm{q} \| \Big\}.
\end{split}
\end{equation*}
And accordingly, $\| \phi(\bm{a})\!-\!\phi(\bm{b}) \|$ can be bounded by Eq.~\eqref{Esim} (see in the next page), which leads to %$\|\cdot\|_{\infty}$ (we omit the subscript notation for simplicity), that is:
\newcounter{mytempeqncnt}
\begin{figure*}[!t]
\normalsize
\begin{equation}\label{Esim}
\begin{split}
 \| \phi(\bm{p})\!-\!\phi(\bm{q}) \|  & \leq \! \frac{1}{\lambda}\bm{K}_+^{-1}\! \Big\{\!\lambda \bm{K}_-(\bm{p}\!-\!\bm{q})\!+\!\frac{1}{n}\bm y \odot \bm{Kh}\! +\! \epsilon_1 \|\bm{p} \|\! -\! \epsilon_2 \|\bm{q} \| \Big\} \\
 &\!\leq\! \big\| \bm{K}_+^{-1}\!\bm{K}_-\!\big\|\!~\!\|\bm{p}\!-\!\bm{q}\|\! + \! \frac{\big\| \bm{K}_+^{-1}\bm{K} \big\|}{\lambda n} \| \bm{h} \| \!+\! \frac{\epsilon_2}{\lambda} \| \bm{K}_+^{-1} \| \!~\! \Big| \|\bm{p} \| - \|\bm{q} \| \Big| \\
 &\!\leq\!\big\| \bm{K}_+^{-1}\big\| \big\|\bm{K}_-\big\|~\|\bm{p}-\bm{q}\| \!+\!
 \frac{\big\| \bm{K}_+^{-1}\bm{K} \big\| \big\|\bm{K} \big\|}{4\lambda n} \|\phi(\bm{p})-\phi(\bm{q})\| + \frac{\epsilon_2}{\lambda} \big\| \bm{K}_+^{-1}\big\| ~\| \bm{p} - \bm{q} \|\,.\\
\end{split}
\end{equation}
\vspace*{-4pt}
\end{figure*}
\begin{equation}\label{final1}
 \| \phi(\bm{p})-\phi(\bm{q}) \| \leq \frac{\|\bm{K}_-\|+\frac{\epsilon_2}{\lambda}}
 {\| \bm{K}_+\|-\frac{1}{4\lambda n} \|\bm{K}\|^2} \| \bm{p} - \bm{q} \|\,.
\end{equation}
Likewise, if $\epsilon_2 < \epsilon_1$, we have
\begin{equation}\label{final2}
 \| \phi(\bm{q})-\phi(\bm{p}) \| \leq \frac{\|\bm{K}_-\|+\frac{\epsilon_1}{\lambda}}
 {\| \bm{K}_+\|-\frac{1}{4\lambda n} \|\bm{K}\|^2} \| \bm{q} - \bm{p} \|\,.
\end{equation}
Accordingly, Eqs.~\eqref{final1} and \eqref{final2} can be reformulated as
\begin{align*}
 \| \phi(\bm{p})-\phi(\bm{q}) \| \leq \frac{\|\bm{K}_-\|+
 \frac{\max\{\epsilon_1,\epsilon_2\}}{\lambda}}
 {\| \bm{K}_+\|-\frac{1}{4\lambda n} \|\bm{K}\|^2} \| \bm{p} - \bm{q} \|\,,
\end{align*}
where we choose $\epsilon = \max\{\epsilon_1,\epsilon_2\}$.
Thereby, the map $\phi$ is a non-expansive mapping with the following condition
\begin{align*}
  \kappa \triangleq \frac{\|\bm{K}_-\|+
 \frac{\epsilon}{\lambda}}
 {\| \bm{K}_+\|-\frac{1}{4\lambda n} \|\bm{K}\|^2} <1\,,
\end{align*}
which derives the upper bound of $\epsilon$ presented in Eq.~\eqref{vare} by some straightforward algebraic manipulations.
Hence, by the fixed point theorem \cite{Goebel1972A}, the map $\phi$ is theoretically demonstrated to be a non-expensive mapping if $\epsilon$ is upper bounded.
%By the fixed point theorem \cite{Goebel1972A}, we can conclude the proof.
\end{proof}
Theorem \ref{theo} demonstrates that the CCICP with early termination condition is  theoretically guaranteed to converge if the inexact parameter $\epsilon$ is upper bounded.
In our IKLR model, such convergence condition is easily satisfied and thus the inexact parameter can be set to a relatively large one in practice.
Note that the early termination condition in Algorithm \ref{alg:one1} is given by variations of the sub-problem function value $\mathcal{F}({\bm \alpha})$ between the two consecutive iterations instead of the gradient variations such as $\|\nabla \mathcal{F}_k({\bm \alpha}_k^{(t+1)}) - \nabla \mathcal{F}_k({\bm \alpha}_k^{(t)})\|<\epsilon$.
This is because it is relatively easier to compute the sub-problem function value than the gradient computation, especially when the stochastic gradient descent method is considered.
\subsection{Analysis with CCICP-SGD}
Apart from an early stop scheme in gradient descent to obtain an inexact solution of the sub-problem, effective stochastic gradient-based methods \cite{Lan2010Optimal,Nitanda2017AISTAS} can also be used as underlying solvers to accelerate the solving process, which achieves the similar approximation in expectation.
Specifically, since the sub-problem becomes strongly convex, the combination of CCCP and SGD in our model achieves fast convergence and theoretical guarantees when compared to directly using SGD to solve the initial non-convex problem.
%Accordingly, we investigate the convergence of CCICP-SGD in generalized version.
%Note that such analysis is not limited to the IKLR model.

%The generalized framework is defined as follows.
%The CCCP or CCICP decomposes the non-convex objective function $f(\bm{\alpha})$ into the difference of two convex functions $g(\bm{\alpha})$ and $h(\bm{\alpha})$, namely $f(\bm{\alpha})=g(\bm{\alpha})-h(\bm{\alpha})$.
%Likewise, the sub-problem is defined as:
%\begin{equation}\label{fk}
% \mathcal{F}_k(\bm{\alpha}) = g(\bm{\alpha})  -[h(\bm{\alpha}_k) + \nabla h^{\top}(\bm{\alpha}_k)(\bm{\alpha}-\bm{\alpha}_k)]\,.
%\end{equation}
%let $F$ denote a convex function over a (closed) convex domain $\mathcal{W}$, which is a subset of some Hilbert space with an induced norm $\| \cdot \|$.
%We assume that $F$ is minimized at some $\bm{\alpha}^* \in \mathcal{W}$.
%Besides general convex $F$, we will also consider the important sub-class of strongly-convex functions.
%Formally, we say that a function $F$ is $\lambda$-\emph{strongly convex}, if for all $\bm{\alpha},~\bm{\alpha}' \in \mathcal{W}$ and any subgradient $\bm{g}$ of $F$ at $\bm{\alpha}$, it holds that
%let $F$ denote a convex function over a (closed) convex domain $\mathcal{W}$, which is a subset of some Hilbert space with an induced norm $\| \cdot \|$.
%We assume that $F$ is minimized at some $\bm{\alpha}^* \in \mathcal{W}$.
Before we prove that a sequence $\{\bm{\alpha}_k \}_{k=0}^{\infty}$ generated by CCICP-SGD converges to a stationary point, we need a few additional results.
We denote the non-convex objective function by $\digamma(\bm \alpha) = g(\bm{\alpha})-h(\bm{\alpha})$.
The idea of the following lemma is to show that the objective function is monotonic decent in probability.
\begin{lemma}\label{lemma2}
  The sequence $\{\bm{\alpha}_k \}_{k=0}^{\infty}$ generated by CCICP-SGD satisfies the monotonic decent property in probability
  \begin{eqnarray}
    \mathbb{E}[\digamma(\bm{\alpha}_{k+1})] &\leq& \mathbb{E}[\digamma(\bm{\alpha}_k)].
  \end{eqnarray}
\end{lemma}
\begin{proof}
Due to the convexity of $g(\bm{\alpha})$ and $h(\bm{\alpha})$, we denote $\hbar(\bm{\alpha})= -h(\bm{\alpha})$ as a concave function, and thus $\digamma(\bm{\alpha})=g(\bm{\alpha})+\hbar(\bm{\alpha})$.
In the $k$th iteration, SGD is used to solve the sub-problem $\mathcal{F}_k(\bm{\alpha})$ in Eq.~\eqref{faak}, yielding $\bm{\alpha}_{k+1}$ to satisfy the KKT condition in expectation, namely $\mathbb{E}[\nabla \mathcal{F}_k(\bm{\alpha}_{k+1})]=0$.
%From Eq.~\eqref{fk}, the gradient of $\mathcal{F}_k(\bm{\alpha})$ with respect to $\bm{\alpha}$ is computed by:
%\begin{equation}\label{gradfk}
%  \nabla \mathcal{F}_k(\bm{\alpha}) = \nabla g(\bm{\alpha}) - \nabla h(\bm{\alpha}_k)\,.
%\end{equation}
%In the current $k$ iteration, $\bm{\alpha}_{k+1}$ generated by SGD satisfies the KKT condition in expectation, namely $\mathbb{E}[\nabla \mathcal{F}_k(\bm{\alpha}_{k+1})]=0$.
Accordingly, we have the following equation
\begin{equation*}
  \mathbb{E}[\nabla g(\bm{\alpha}_{k+1})] = \mathbb{E}[\nabla h(\bm{\alpha}_k)] = -\mathbb{E}[\nabla \hbar(\bm{\alpha}_k)]\,.
\end{equation*}
For $\bm{\alpha}_k$ and $\bm{\alpha}_{k+1}$, it satisfies
%\begin{equation}
%\left\{
%\begin{array}{rcl}
%\begin{split}
%&g(\bm{\alpha}_2) \geq g(\bm{\alpha}_1) + (\bm{\alpha}_2-\bm{\alpha}_1)\nabla g(\bm{\alpha}_1) \\
%&\hbar(\bm{\alpha}_4) \leq \hbar(\bm{\alpha}_3) + (\bm{\alpha}_4-\bm{\alpha}_3)\nabla h(\bm{\alpha}_3)\\
%\end{split}
%\end{array} \right.
%\,.
%\end{equation}
%Set $\bm{\alpha}_1=\bm{\alpha}_4 \triangleq \bm{\alpha}_{k+1}$, and $\bm{\alpha}_2 = \bm{\alpha}_3 \triangleq \bm{\alpha}_{k}$, we have:
\begin{equation}\label{asdd}
\left\{
\begin{array}{rcl}
\begin{split}
&  g(\bm{\alpha}_k) \geq g(\bm{\alpha}_{k+1}) + (\bm{\alpha}_k-\bm{\alpha}_{k+1})\nabla g(\bm{\alpha}_{k+1}) \\
&\hbar(\bm{\alpha}_k) \geq \hbar(\bm{\alpha}_{k+1}) + (\bm{\alpha}_k-\bm{\alpha}_{k+1})\nabla \hbar(\bm{\alpha}_k) \,.
\end{split}
\end{array} \right.
\end{equation}
Then we take the expectation of Eq.~\eqref{asdd} and use the expectation independence
\begin{equation*}
\left\{
\begin{array}{rcl}
\begin{split}
& \! \! \! \mathbb{E}[g(\bm{\alpha}_k)] \geq \mathbb{E}[g(\bm{\alpha}_{k+1})] + \mathbb{E}[(\bm{\alpha}_k-\bm{\alpha}_{k+1})]\mathbb{E}[\nabla g(\bm{\alpha}_{k+1})] \\
&\! \! \!   \mathbb{E}[\hbar(\bm{\alpha}_k)] \geq \mathbb{E}[\hbar(\bm{\alpha}_{k+1})] + \mathbb{E}[(\bm{\alpha}_k-\bm{\alpha}_{k+1})]\mathbb{E}[\nabla \hbar(\bm{\alpha}_k)]  \,.
\end{split}
\end{array} \right.
\end{equation*}
Combining the above two sub-equations, we have
\begin{eqnarray}
  \mathbb{E}[g(\bm{\alpha}_k)+\hbar(\bm{\alpha}_k)] &\geq& \mathbb{E}[g(\bm{\alpha}_{k+1})+\hbar(\bm{\alpha}_{k+1})]\,,
\end{eqnarray}
which completes the proof.
\end{proof}
Lemma \ref{lemma2} demonstrates the monotonic decent property in probability.
However, such analysis is not complete, as the monotone descent property by itself is not sufficient to claim the convergence of $\{\bm{\alpha}_k \}_{k=0}^{\infty}$.
The similar situation in the initial CCCP version has been discussed in \cite{Sriperumbudur2009On}.

In the following section, we firstly give the definition of Lipschitz smoothness required by the subsequent analyses, and then investigate the convergence of $\{\bm{\alpha}_k \}_{k=0}^{\infty}$.
\begin{definition}
  A function $F$ is gradient Lipschitz smooth if there exists $L_F>0$ such that
  \begin{equation*}\label{Lip}
   \| \nabla F(\bm{\alpha}) -  \nabla F(\bm{\alpha}') \|_2 \leq L_F\|\bm{\alpha} - \bm{\alpha}'   \|_2, \forall \bm{\alpha}, \bm{\alpha}' \in \text{dom}~ F.
  \end{equation*}
\end{definition}
Further, if $F$ is also a convex function over a (closed) convex domain, it satisfies
\begin{equation*}
  F(\bm{\alpha})-F(\bm{\alpha}') \leq \nabla F^{\top}(\bm{\alpha}')(\bm{\alpha} -\bm{\alpha}') + \frac{L_F}{2}\|  \bm{\alpha} -  \bm{\alpha}' \|_2^2\,.
\end{equation*}
Suppose that $\mathcal{F}_k$ is $L_{\mathcal{F}}$-smooth function, we have
\begin{equation*}
  \mathcal{F}_k(\bm{\alpha})-\mathcal{F}_k(\bm{\alpha}_k) \leq \nabla \mathcal{F}^{\top}_k(\bm{\alpha}_{k})(\bm{\alpha}-\bm{x}_k) +\frac{L_{\mathcal{F}}}{2}\| \bm{\alpha} - \bm{\alpha}_k \|_2^2\,.
\end{equation*}
To obtain $\bm{\alpha}_{k+1}$,  we use SGD to solve $\mathcal{F}_k(\bm{\alpha})$ with $T$ iterations, that is $ \bm{\alpha}_{k+1} \triangleq \bm{\alpha}_k^{(T)}$.
Specifically, the first iteration is defined as $\bm{\alpha}_k^{(1)}=\bm{\alpha}_k - \eta \widehat{\bm{g}}$ where the stepsize $\eta$ is set to $\frac{1}{L_{\mathcal{F}}}$ and the produced vector $\widehat{\bm{g}}$ satisfies $\mathbb{E}[\widehat{\bm{g}}]=\nabla \mathcal{F}_k(\bm{\alpha}_k)$.
Therefore, we have
\begin{equation}\label{as}
\begin{split}
&\mathcal{F}_k(\bm{\alpha}_k^{(1)})\!-\!\mathcal{F}_k(\bm{\alpha}_k)\! \leq \!\nabla \mathcal{F}^{\top}_k(\bm{\alpha}_{k})(\bm{\alpha}_k^{(1)}\!-\!\bm{\alpha}_k)\! +\!\frac{L_{\mathcal{F}}}{2}\|  \bm{\alpha}_k^{(1)}\!\! -\!\! \bm{\alpha}_k \|_2^2 \\
&~~~~~~~~~~~~~~~~~~~~~~= -\frac{1}{L_{\mathcal{F}}} \nabla \mathcal{F}^{\top}_k(\bm{\alpha}_{k}) \widehat{\bm{g}} + \frac{1}{2L_{\mathcal{F}}} \| \widehat{\bm{g}} \|_2^2\,.
  \end{split}
\end{equation}
We take the expectation
\begin{equation}\label{ad}
  \mathbb{E}[\mathcal{F}_k(\bm{\alpha}_k^{(1)})]\! \leq \! \mathbb{E}[\mathcal{F}_k(\bm{\alpha}_k)] - \frac{1}{L_{\mathcal{F}}} \mathbb{E}[\nabla \mathcal{F}^{\top}_k(\bm{\alpha}_{k}) \widehat{\bm{g}}] + \frac{1}{2L_{\mathcal{F}}} \mathbb{E}[\| \widehat{\bm{g}} \|_2^2],
\end{equation}
where the formula satisfies $\mathbb{E}[\nabla \mathcal{F}^{\top}_k(\bm{\alpha}_{k}) \widehat{\bm{g}}] \!= \! \mathbb{E}[\nabla \mathcal{F}^{\top}_k(\bm{\alpha}_{k})]\mathbb{E}[\widehat{\bm{g}}]$ because they are independent of each other, and $\mathbb{E}[\mathcal{F}_k(\bm{\alpha}_{k+1})] \triangleq \mathbb{E}[\mathcal{F}_k(\bm{\alpha}_k^{(T)})] \leq \mathbb{E}[\mathcal{F}_k(\bm{\alpha}_k^{(1)})]$ by Lemma \ref{lemma2}.
By combining Eq.~\eqref{as} and Eq.~\eqref{ad}, we obtain
\begin{equation}\label{expef}
  \mathbb{E}[\mathcal{F}_k(\bm{\alpha}_{k+1})] \leq  \mathbb{E}[\mathcal{F}_k(\bm{\alpha}_{k})] - \frac{1}{2L_{\mathcal{F}}} \mathbb{E}[\| \nabla \mathcal{F}^{\top}_k(\bm{\alpha}_{k}) \|_2^2]\,.
\end{equation}

%The gradient of $f$ is Lipschitz continuous with parameter $L>0$ if
%\begin{equation}
%  \| \nabla f(x) -  \nabla f(y) \|_2 \leq L\|x-y  \|_2, \forall x,y \in \text{dom}~ f.
%\end{equation}
%Therefore,
%\begin{equation}
%  f(x)-f(y) \leq \nabla f^{\top}(x)(x-y) -\frac{1}{2L}\| \nabla f(x) - \nabla f(y) \|_2^2
%\end{equation}
%Define
%\begin{equation}
%  \phi_k(x) = g(x) + \frac{\mu}{2}\|x-x_k\|_2^2 -[h(x_k) + \nabla h^{\top}(x_k)(x-x_k)]
%\end{equation}
%Since $\phi_k$ is $(L+\mu)$ smooth function, we have,
%\begin{equation}
%  \phi_k(x_{k+1})-\phi_k(x_k) \leq \nabla \phi^{\top}_k(x_{k+1})(x_{k+1}-x_k) -\frac{1}{2L+2\mu}\| \nabla \phi_k(x_{k+1}) - \nabla \phi_k(x_k) \|_2^2
%\end{equation}
%Suppose that $\| \nabla \phi_k(x_{k+1}) - \nabla \phi_k(x_k) \|_2^2 \geq C \|\nabla \phi_k(x_{k}) \|_2^2$, and take the expectation,
%\begin{equation}
%  \mathbb{E}[\phi_k(x_{k+1})] \leq \mathbb{E}[\phi_k(x_{k})] + \frac{1}{2} \mathbb{E}[\| \nabla \phi_k(x_{k+1}) \|^2_2] + \frac{1}{2}\mathbb{E}[\|x_{k+1}-x_k \|_2^2] - \frac{C}{2(L+\mu)}\mathbb{E}[\|\nabla \phi_k(x_{k}) \|_2^2]
%\end{equation}
%Note that $\phi_k(x_k)=f(x_k)$, $\nabla \phi_k(x_k) =\nabla f(x_k) $ and $\mathbb{E}(\| \nabla \phi_k(x_{k+1}) \|^2_2 ) \leq \mathbb{E}(\| \nabla \phi_k(x_{k}) \|^2_2 )$,
%\begin{equation}
%  \mathbb{E}[\phi_k(x_{k+1})] \leq \mathbb{E}[f(x_{k})] + \frac{1}{2}(1-\frac{C}{L+\mu}) \mathbb{E}[\| \nabla \phi_k(x_{k}) \|^2_2] + \frac{1}{2}\mathbb{E}[\|x_{k+1}-x_k \|_2^2]
%\end{equation}

Besides, $h(\bm{\alpha})$ is a convex function, and it satisfies $h(\bm{\alpha}_k)-h(\bm{\alpha}_{k+1}) \leq \nabla h^{\top}(\bm{\alpha}_k)(\bm{\alpha}_k-\bm{\alpha}_{k+1})$, and then
\begin{equation}\label{expeff}
\begin{split}
&~~~~\mathbb{E}[f(\bm{\alpha}_{k+1})] =  \mathbb{E}[g(\bm{\alpha}_{k+1}) -h(\bm{\alpha}_{k+1}) ]\\
&\leq \mathbb{E}\big[g(\bm{\alpha}_{k+1}) -  h(\bm{\alpha}_k) + \nabla h^{\top}(\bm{\alpha}_k)(\bm{\alpha}_k-\bm{\alpha}_{k+1})\big] \\
&= \mathbb{E}[\mathcal{F}_k(\bm{\alpha}_{k+1})]\,.
  \end{split}
\end{equation}
%Therefore,
%\begin{equation}
%  \frac{1}{2}(\frac{C}{L+\mu}-1) \mathbb{E}[\| \nabla \phi_k(x_{k}) \|^2_2] \leq \mathbb{E}[f(x_{k+1} - f(x_{k})] + \frac{1}{2}(1-\mu)\mathbb{E}[\|x_k-x_{k+1}\|_2^2]
%\end{equation}
%Let $\mu > 1$, and thus,
%\begin{equation}
%  \frac{1}{2}(\frac{C}{L+\mu}-1) \mathbb{E}[\| \nabla \phi_k(x_{k}) \|^2_2] \leq \mathbb{E}[f(x_{k+1} - f(x_{k})]
%\end{equation}
Note that $f(\bm{\alpha}_k)=\mathcal{F}_k(\bm{\alpha}_{k})$, from Eqs.~\eqref{expef} and \eqref{expeff}, we obtain
\begin{equation*}
  \mathbb{E}[f(\bm{\alpha}_{k+1})] \leq \mathbb{E}[f(\bm{\alpha}_{k})] - \frac{1}{2L_{\mathcal{F}}} \mathbb{E}[\| \nabla f(\bm{\alpha}_{k}) \|_2^2]\,.
\end{equation*}
Finally, we arrive at the following formula as we expect
\begin{equation}
  \mathbb{E}[f(\bm{\alpha}_{k}) - f(\bm{\alpha}_{k+1})] \geq \frac{1}{2L_{\mathcal{F}}} \mathbb{E}[\| \nabla f(\bm{\alpha}_{k}) \|_2^2]\,.
\end{equation}
By Lemma \ref{lemma2} and the above formula, we verify the convergence of $\{\bm{\alpha}_k \}_{k=0}^{\infty}$.
Such proof is definitely suitable for the proposed IKLR model.
To be specific, $L_{\mathcal{F}}$ in our IKLR model can be solved during the proof of Theorem \ref{theo}, and thus it is set to $L_{\mathcal{F}} = \lambda\| \bm{K}_+\|+\frac{\| \bm{K}\|}{4n}$.

{\bf Remark:} The convergence analyse of a stochastic proximal difference of convex algorithm has been discussed in \cite{Nitanda2017AISTAS}, which relies on
a known bounded residual error $\delta$, namely, $\mathcal{F}_k(\bm{\alpha}_{k+1}) \leq \mathcal{F}_k(\bm{\alpha}_{k}^*)+\delta$ as demonstrated.
However, the residual error $\delta$ is usually not known. In our analysis, the decrease is related to the gradient, which could be calculated or estimated in each iteration.

\section{Experiments}
\label{sec:experiment}
In this section, we carry out experiments to show the performance of the IKLR model with two indefinite kernels on a collection of multi-modal data sets from computer vision and machine learning fields.
The experiments implemented in MATLAB are repeated over 10 runs on a PC with Intel i5-6500 CPU (3.20 GHz) and 8 GB memory.
The source code of the proposed method can be found in \url{http://www.lfhsgre.org}.
\subsection{Experiment Setup}
Here we describe kernel settings, the compared algorithms, and other settings of the experiments.
%The MNIST database of handwritten digits (0$\sim$9) has a training set of 60,000 examples, and a test set of 10,000 examples.
%The digits have been size-normalized and centered in a fixed-size image (28$\times$28).

\subsubsection{Kernel setting}
Three kernels including a positive definite one and two indefinite ones are chosen to fully evaluate the performance of our method.
As a representative PD kernel, the RBF (Radius Basis Function) kernel, \emph{i.e.}
$\mathcal{K}(\bm{x}_i,\bm{x}_j)=\exp(-\| \bm{x}_i-\bm{x}_j\|^2_2/\sigma^2)$ with the kernel width $\sigma$, is chosen for comparison.
%The spread parameter $\sigma$ is determined by five-fold cross-validation on the training set.

For indefinite kernels, we first choose the truncated $\ell_1$ distance (TL1) indefinite kernel \cite{huang2017classification}, namely $\mathcal{K}(\bm{x}_i,\bm{x}_j)= \max\{\tau-\|\bm{x}_i-\bm{x}_j\|_1,0\}$, and then incorporate it into our model.
As discussed in \cite{huang2017classification}, the performance of the TL1 kernel is robust to $\tau$, and thus we set $\tau = 0.7m$ as suggested.

Apart from a delicately designed TL1 kernel, we extend the RBF kernel from Euclidean space to a Riemannian manifold with a geodesic metric \cite{Jayasumana2013Kernel}.
%However, not all geodesic metric yield a positive definite kernels.
Here we use the covariance matrix descriptor \cite{Tuzel2006Region} on the space of $d \times d$ symmetric positive definite (SPD) matrices, namely $Sym_d^+$.
%Let $\bm{S}_1$ and $\bm{S}_2$ be two descriptors (SPD matrices), different ways to measure the similarity between $\bm{S}_1$ and $\bm{S}_2$ result in different properties.
%The relationship between the metric (including Euclidean and geodesic distance) and the derived Gaussian kernel on Riemannian manifold can be found in \cite{Feragen2015Geodesic}.
Let $\bm{S}_1$ and $\bm{S}_2$ be two descriptors (SPD matrices), if the Euclidean distance in Gaussian kernel between such two descriptors is used, the derived Gaussian kernel is positive definite.
Comparably, to define a kernel on a Riemannian manifold, we would like to replace the Euclidean distance
by a more accurate geodesic distance on the manifold.
However, not all geodesic metrics yield a PD kernel.
In \cite{Feragen2015Geodesic}, the authors point out that the geodesic
Gaussian kernels on Riemannian manifolds are PD only if the geodesic metric space is \emph{flat in the sense of Alexandrov}.
Here we summarize the definitions and properties of some representative metrics for $Sym_d^+$ in Table \ref{PSDkernel}.
It can be observed that the geodesic Gaussian kernel (``Log-Euclidean") is still PD since the geodesic metric space derived by Log-Euclidean is \emph{flat}, while the affine-invariant metric results in an indefinite one.
Based on this, in our experiment, we take the affine-invariant kernel as an example of indefinite kernels to test the proposed IKLR model.
\begin{table*}
\centering
\small
%\scriptsize
\caption{Properties of different metrics on $Sym_d^+$. We analyze positive definiteness of Gaussian kernels generated by different metrics.}
\begin{tabular}{cccccccc}
\toprule
    & \multirow{2}{*}{Metric}  &\multirow{2}{*}{Formulation ($d$)} &\multirow{2}{*}{Geodesic distance?} &Does $k=\exp(-d^2/\sigma^2)$ & \multirow{2}{*}{Time complexity}\\
    & & & & define a positive definite kernel? &\\
  \midrule
  & Euclidean & $\| \bm{S}_1 - \bm{S}_2 \|_\text{F}$ & No  & Yes & $\mathcal{O}(d^2)$\\
  \hline
  &Log-Euclidean (Log-E) &$\|\log(\bm{S}_1)-\log(\bm{S}_2)\|_\text{F}$ & Yes  &Yes & $\mathcal{O}(d^3)$\\
  \hline
  &Affine-Invariant (Aff-I) &$\| \log(\bm{S}_1^{-\frac{1}{2}}\bm{S}_2\bm{S}_1^{-\frac{1}{2}}) \|_\text{F} $	&Yes	& No  & $\mathcal{O}(d^3)$ \\
\bottomrule
\end{tabular}
\label{PSDkernel}
\end{table*}

\subsubsection{Compared methods}
\label{sec:compare}
We conduct the proposed IKLR model with two versions, including CCICP-GD and CCICP-SGD, in which the inexact parameter $\epsilon$ in CCICP-GD and CCICP-SGD is set to 1 and 0.0001, respectively\footnote{Since SGD achieves the similar purpose with the inexact scheme to solve the sub-problem, the inexact parameter $\epsilon$ in CCICP-SGD is fixed with an ``exact" one.}.
The proposed algorithms are compared with other representative indefinite kernel methods:
``Flip", ``Clip", and ``Shift" \cite{wu2005analysis}: these methods use the spectrum transformation to directly transform the non-PSD kernel matrix into a PSD one;
``TDCASVM" \cite{Akoa2008Combining}: an approach incorporates CCCP into the SMO-type algorithm for SVM with an indefinite kernel;
``KSVM" \cite{Ga2016Learning}: this method formulates the indefinite SVM as a min-max problem, and then solves this optimization problem in Kre\u{\i}n space.
In essence, our method and ``TDCASVM" directly solve a non-convex optimization problem, while the objective function in other algorithms has been transformed into a convex form.
Specifically, two PD kernels, \emph{i.e.} Gaussian kernel and Log-Euclidean kernel are incorporated into SVM and KLR as baseline methods for comparisons.

\subsubsection{Parameter setting}
In our experiment, we choose $\lambda$, the kernel width $\sigma$ in Gaussian kernel, and $C$ in SVM by five-fold cross validation over $\{ 0.0001, 0.001, 0.01, 0.1, 1, 5,10\}$ on the training set.
%: one of these five subsets is used for validation in turn and the remaining ones for training.
%The inexact parameter $\epsilon$ is set to 1.
For each dataset, half of the data are randomly picked up for training and the rest for the test process.

\subsection{Results on UCI Database}
\label{sec:uci}
\subsubsection{Description of data sets}
Table \ref{tab:freq} lists a brief description of 20 data sets from UCI Machine Learning Repository \cite{Blake1998UCI} including the feature dimension $m$, the number of data points $n$ (the training and test data have been divided in some datasets such as \emph{monks1}, \emph{monks2}, and \emph{monks3}), the minimum and maximum eigenvalues $\mu_{\min}$ and $\mu_{\max}$ of the TL1 kernel.
\begin{table}
\centering
  \caption{%Statistics for various data sets with $n$ training samples represented by a $m$-dimensional feature. The notations $\lambda_{\max}$ and $\lambda_{\min}$ denote the maximum and minimum eigenvalues of the TL1 kernel over training samples. The large-scale data sets are highlighted by \textbf{bold}.
  Statistics for various data sets. Specifically, the larger than small scale data sets are highlighted by \textbf{bold}.
  }
  \label{tab:freq}
  \begin{tabular}{ccccccc}
    \toprule
  Dataset  &$m$(feature) 	&$n$(\#num) & $\mu_{\min}$ &$\mu_{\max}$\\
  \midrule
  australian
  &14
  &690
  &-0.006
 &2212.5\\
 \hline
  breast-cancer
  &10
  &699
  &-4.408
 &1593.3\\
 \hline
 parkinsons
  &23
  &195
  &0.127
 &1200.4\\
   \hline
   climate
  &20
  &540
  &0.174
 &1944.7\\
 \hline
  diabetic
  &19
  &1151
  &-0.003
 &6133.1\\
 \hline
   fertility
  &9
  &100
  &-0.042
 &164.98\\
 \hline
 sonar
     &60
  &208
  &1.452
 &3024.6\\
   \hline
SPECT
  &21
  &80
  &-1.145
 &353.11\\
   \hline
  haberman
  &3
  &306
  &-0.204
 &215.23\\
 \hline
 heart
  &13
  &270
  &-0.084
 &695.16\\
 \hline
 ionosphere
   &33
  &351
  &0.085
 &2489.7\\
 \hline
monks1
  &6
  &124
  &-2.094
 &94.077\\
  \hline
  monks2
  &6
  &169
  &-2.535
 &131.14\\
   \hline
  monks3
  &6
  &122
  &-1.764
 &95.376\\
 \hline
   splice
   &60
  &1000
  &-1.325
 &2885.3\\
 \hline
transfusion
  &4
  &748
  &-0.336
 &818.74\\
 \hline
 {\bf EEG}
  &14
  &14980
  &-0.444
 &7312.0\\
  \hline
   {\bf ijcnn1-tr}
  &26
  &35000
  &-0.018
 &28945\\
    \hline
 {\bf guide1-t}
  &4
  &4000
  &-0.805
 &4116.7\\
   \hline
 {\bf madelon}
  &500
  &2000
  &14.825
 &27015\\
  \bottomrule
\end{tabular}
\end{table}
%It can be observed that the absolute value of the maximum eigenvalue in each data set is always much larger than that of the minimum one, which means that the CCICP will possess fast in our IKLR model as discussed in Section~\ref{sec:convr}.

%We compare IKLR with other representative state-of-the-art indefinite kernel learning based algorithms including:
%``Flip", ``Clip", and ``Shift" \cite{Wu2005An}: three methods directly convert the indefinite kernel matrix generated by TL1 kernel into a positive semi-definite matrix using the spectrum transformation.
%Then we take the modified kernel matrix into kernel logistic regression.
%``SVM(RBF)": a representative classification method uses SVM with the RBF kernel.
%``KSVM" \cite{Ga2016Learning}: a method transforms TL1 kernel from RKKS to RKHS, and then trains the convex dual form of SVM.
%``I-LSSVM" \cite{Huang2016Indefinite}: a method uses the TL1 kernel and then solves with LS-SVM.
%``KLR"  \cite{Zhu2002Kernel}: a representative classification method uses logistic regression with the RBF kernel just for self-verification.
\subsubsection{Results on small-scale data sets}
Table~\ref{UCIres} reports the average classification accuracy on the test data and its standard deviation.
%The best classification accuracy on each dataset in the sense of average accuracy is highlighted in bold.
\begin{table*}[!htb]
\centering
\small
%\scriptsize
\caption{Classification accuracy (mean$\pm$std. deviation) of each compared method. The best performance is highlighted in \textbf{bold}. }
%\begin{tabular}{ccccccccc}
\begin{tabular}{|p{1.2cm}|p{1.5cm}|p{1.5cm}|p{1.5cm}|p{1.5cm}|p{1.5cm}|p{1.5cm}|p{1.5cm}|p{1.7cm}|}
%\toprule
\hline
    &KLR(RBF) & Flip  &Clip &Shift &TDCASVM & KSVM &CCICP-GD &CCICP-SGD \\
  %\midrule
  \hline
    australian
  &0.790$\pm$0.021 &0.795$\pm$0.059	&0.790$\pm$0.078	&0.672$\pm$0.085 &0.857$\pm$0.013 &0.835$\pm$0.018	&0.846$\pm$0.027	&{\bf 0.865}$\pm$0.007\\
  \hline
  breast
    &0.968$\pm$0.004 &0.964$\pm$0.012	&0.965$\pm$0.007	&0.921$\pm$0.039		&0.946$\pm$0.008 &{\bf 0.971}$\pm$0.008	&0.959$\pm$0.014	&0.967$\pm$0.008 \\
    \hline
    climate
&	{\bf 0.939}$\pm$0.012 &	0.909$\pm$0.047	&	0.838$\pm$0.036	&	0.839$\pm$0.061	&0.904$\pm$0.031 &	0.918$\pm$0.011	&	0.912$\pm$0.013&	0.923$\pm$0.018	\\
\hline
diabetic	
&	{\bf 0.625}$\pm$0.007 &	0.532$\pm$0.016&	0.529$\pm$0.012	&	0.534$\pm$0.016	&0.545$\pm$0.024 &	0.570$\pm$0.020	&	0.552$\pm$0.034	&	0.516$\pm$0.065	\\	\hline
fertility	&	0.852$\pm$0.023 &	{\bf 0.896}$\pm$0.025	&	0.888$\pm$0.039&	0.884$\pm$0.040	&0.880$\pm$0.025 &	0.873$\pm$0.034	&	0.780$\pm$0.040 &	0.860$\pm$0.024	\\	\hline
haberman	&	0.742$\pm$0.040 &	0.678$\pm$0.039&	0.725$\pm$0.037&	0.733$\pm$0.023	&0.722$\pm$0.031 &	0.730$\pm$0.034	&	0.727$\pm$0.035	&	\textbf{0.766}$\pm$0.020\\	\hline
heart	&	0.816$\pm$0.044 &	0.758$\pm$0.076&	0.760$\pm$0.035	&	0.747$\pm$0.059	&\textbf{0.823}$\pm$0.031 &	0.757$\pm$0.043	&	0.803$\pm$0.046	&	0.809$\pm$0.023	\\	\hline
ionosphere	&	0.907$\pm$0.029 &	0.891$\pm$0.032&	0.900$\pm$0.032	&	0.841$\pm$0.022	&0.903$\pm$0.022 &	0.899$\pm$0.014	&	0.901$\pm$0.015	&	\textbf{0.915}$\pm$0.028	\\	\hline
monks1	&	0.668$\pm$0.052 &	0.695$\pm$0.075	&	0.648$\pm$0.070	&	0.685$\pm$0.063	&0.678$\pm$0.021 &	0.671$\pm$0.035	&	\textbf{0.765}$\pm$0.065	&	0.752$\pm$0.041	\\	\hline
monks2	&	0.662$\pm$0.071 &	0.611$\pm$0.047	&	0.593$\pm$0.025	&	0.489$\pm$0.092	&\textbf{0.743}$\pm$0.018 &	0.626$\pm$0.037	&	0.669$\pm$0.093	&	0.617$\pm$0.046	\\	\hline
monks3	&	0.779$\pm$0.073 &	0.723$\pm$0.090	&	0.805$\pm$0.021	&	0.870$\pm$0.036	&0.734$\pm$0.045 &	0.640$\pm$0.083	&	0.830$\pm$0.072	&	\textbf{0.893}$\pm$0.031	\\	\hline
parkinsons	&	\textbf{1.000}$\pm$0.000 &	0.990$\pm$0.010	&	0.999$\pm$0.003	&	0.998$\pm$0.007	&	\textbf{1.000}$\pm$0.000 &	0.945$\pm$0.039	&	\textbf{1.000}$\pm$0.000	&	\textbf{1.000}$\pm$0.000	\\	\hline
sonar	&	0.789$\pm$0.022 &	0.546$\pm$0.045	&	0.539$\pm$0.042	&	0.504$\pm$0.054	&0.704$\pm$0.024 &	0.608$\pm$0.072	&	\textbf{0.794}$\pm$0.060	&	0.690$\pm$0.121	\\	\hline
SPECT	&	0.737$\pm$0.092 &	0.652$\pm$0.026	&	0.706$\pm$0.022	&	0.667$\pm$0.034	&0.711$\pm$0.024 &	\textbf{0.893}$\pm$0.024	&	0.764$\pm$0.059	&	0.738$\pm$0.076	\\	\hline
splice	&	0.642$\pm$0.093 &	0.513$\pm$0.017	&	0.619$\pm$0.057	&	0.604$\pm$0.033	&0.751$\pm$0.075 &	0.515$\pm$0.029	&	\textbf{0.785}$\pm$0.050	&	0.588$\pm$0.047	\\	\hline
transfusion	&	0.741$\pm$0.048 &	0.734$\pm$0.095	&	0.717$\pm$0.020	&	0.736$\pm$0.038	&0.762$\pm$0.015 &	0.762$\pm$0.006	&	0.769$\pm$0.023	&	\textbf{0.774}$\pm$0.009	\\	
 %\bottomrule
 \hline
\end{tabular}
\label{UCIres}
\end{table*}
One can see that the proposed IKLR model with CCICP-GD and CCICP-SGD algorithms achieves a promising performance in most data sets such as \emph{australian}, \emph{monks1}, and \emph{splice}.
TDCASVM, KSVM and the baseline (KLR) also provide a comparable performance on some data sets including \emph{monks2}, \emph{SPECT}, and \emph{diabetic}.
Meanwhile, the performance of kernel approximation methods is often inferior to other indefinite learning based algorithms.
Besides, we observe that the training TL1 kernel in several data sets such as \emph{climate} and \emph{parkinsons} is still positive definite.
In these data sets, there is no distinct difference on the classification accuracy for most compared algorithms.
%For example, in \emph{ionosphere} dataset, there is only 2.5\% gap between the best performance and the worst one.
%Specifically, compared to logistic regression with positive definite kernels (e.g. RBF), the proposed IKLR algorithm shows a superiority of indefinite kernel learning in terms of classification results on these ten datasets.

In terms of these algorithms, comparing the baseline (KLR), we find that the used TL1 kernel is able to adaptively find the partition and locally fit nonlinearity.
However, in this paper we do not want to claim that the TL1 kernel is better than the RBF kernel, as their performance is actually dependent on the specific task.
Instead, our aim is to show the performance of the indefinite kernels in KLR.
Since the indefinite kernels contain definite ones, it can be expected that a suitable indefinite kernel can outperform a positive definite kernel.

%But kernel design and training are beyond the scope of this paper.
%Besides, the ``flip", ``clip", and ``shift" methods actually change the indefinite matrix itself, which results in the inconsistency between the training and the test kernel.
%The inconsistency happens since the above operators can only be used for training data.
%Comparably, our method does not change the kernel for neither the training nor test data, and thus it reserves the implicit information involved with the indefinite kernel to achieve better performance.
Besides, it can be noted that KSVM investigates its dual form in RKKS, while we directly focus on the non-convex primal form of the IKLR model using the representer theorem.
The coefficient vector ${\bm \alpha}$ in IKLR should not be interpreted as a Lagrange multiplier, which is different from the dual variable in SVM.
Therefore, our IKLR model is more flexible to learn the data distribution than KSVM.
Admittedly, KLR and SVM based algorithms including TDCASVM and KSVM have their respective pros and cons, and each solution might depend on a case by case basis.

\subsubsection{Results on larger than small scale data sets}
Apart from experiments on several small scale data sets, we also conduct the proposed CCICP algorithm on four larger than small scale data sets to further validate its effectiveness.
 Table~\ref{tabptb} reports the test accuracy, training time and test time of the original CCCP, and the proposed CCICP-GD and CCICP-SGD.
Specifically, CCCP-GD is taken as a baseline to evaluate the performance of the proposed two algorithms.
In CCCP-GD, the inexact parameter $\epsilon$ is fixed with 0 in theoretical aspect while we set it to 0.0001 in practice.
This small value in our experiments means that CCCP-GD is not equipped with any inexact scheme, which further guarantees that CCCP-GD is able to yield an accurate solution during each iteration.
\begin{table}[t]
\centering
  \begin{threeparttable}
\scriptsize
\caption{\footnotesize Results of CCCP-GD, CCICP-GD and CCICP-SGD on four larger than small scale data sets.}\label{tabptb}
%\begin{tabular}{|*{13}{l|}}
\begin{tabular}{ccccccccccc}
  %\Xhline{1.0pt}
  \toprule
  Data sets & Methods & Accuracy & Training time(s) & Test time(s)\\
  \midrule
  \multirow{3}{*}{EEG} &{CCCP-GD} &0.769$\pm$0.042 &17171.0 &0.1237 \\
  &CCICP-GD &0.725$\pm$0.042 &848.89 &0.1304  \\
  &CCICP-SGD &0.730$\pm$0.080 &8776.4  &0.1188  \\
  \hline
  \multirow{3}{*}{guide1-t} &{CCCP-GD} &0.962$\pm$0.003 &1314.3 &0.0020 \\
  &CCICP-GD &0.955$\pm$0.003 &47.23  &0.0028  \\
  &CCICP-SGD &0.947$\pm$0.022 &714.72   &0.0021  \\
  \hline
    \multirow{3}{*}{ijcnn1-tr} &{CCCP-GD} &0.912$\pm$0.002 &51.22  &0.2215 \\
  &CCICP-GD &0.914$\pm$0.003 &14.65   &0.2273   \\
  &CCICP-SGD &0.914$\pm$0.006 &28.23    &0.2258  \\
  \hline
  \multirow{3}{*}{madelon} &{CCCP-GD} &0.624$\pm$0.080 &305.29  &0.0008 \\
  &CCICP-GD &0.609$\pm$0.051 &8.129   &0.0064   \\
  &CCICP-SGD &0.601$\pm$0.028 &160.71    &0.0011  \\
  \bottomrule
\end{tabular}
 \end{threeparttable}
\end{table}

On \emph{EEG}, \emph{guide1-t}, and \emph{madelon} data sets, CCCP-GD achieves the best performance on classification accuracy, which is narrowly followed by CCICP-GD and CCICP-SGD.
On the \emph{ijcnn1-tr} data set, the above three algorithms achieve the similar classification performance without distinct difference.
In terms of the computational cost during training, CCICP-GD is the most efficient, while CCCP-GD is much time-consuming.
Note that $\epsilon$ in the proposed CCICP-SGD algorithm is set to 0.0001, which makes the training process relatively inefficient.
Specifically, we also discuss CCICP-SGD with $\epsilon=1$ in Section \ref{sec:diss} to see how fast it is.
From above analyses, we can conclude that the proposed IKLR model with CCICP is often slightly inferior to the CCCP setting in the terms of classification accuracy, but the inexact scheme makes our method much efficient during the training process.
%However, observe that, CCICP-GD spends about twice as much time cost as CCCP-GD on \emph{ijcnn1-tr}.
%The reduction of computational load in CCICP-GD on this data set does not show a consistency with that of other three ones.
%We will illustrate this case by providing the convergence analysis demonstrated in Section \ref{sec:conv}.

The experimental results on UCI database demonstrate the superiority of our IKLR model with positive definite or indefinite kernels.
Besides, the inexact scheme including the early termination condition and a stochastic version is able to accelerate the training process.
%Above results on various data sets demonstrate that the proposed IKLR model not only outperforms the non-convex optimization and kernel approximation with a statistically significant evidence on the indefinite training kernel, but also achieves a favorable classification accuracy on the training data set with positive definite kernels.
%Moreover, the inexact scheme including the early termination condition and a stochastic version can effectively speed up the training process.
%Further, designing an advanced and delicate kernel in kernel logistic regression is more flexible to achieve promising performance, not limited to a PD kernel.
\subsection{Results on Yale Face Database}
Apart from using the designed TL1 kernel on UCI database, we also illustrate the use of the geodesic Gaussian kernel on the Riemannian manifold in the IKLR model for face recognition on $Sym_d^+$.
In the experiment, we choose the Yale face database B\footnote{\url{http://vision.ucsd.edu/content/yale-face-database}} to evaluate the performance of the proposed IKLR model with the Affine-Invariant kernel.
This database contains 5760 single light source images of 10 subjects with each shot under 576 viewing conditions (9 poses $\times$ 64 illumination conditions).
For every subject in a particular pose, an image with ambient (background) illumination is also captured. Hence, the total number of images is in fact 5760+90=5850.
All images have been cropped based on the location of eyes. The size of each image is 192$\times$168.
Fig. \ref{yaleB} shows some image examples of this database.
\begin{figure}
\begin{center}
\includegraphics[width=0.36\textwidth]{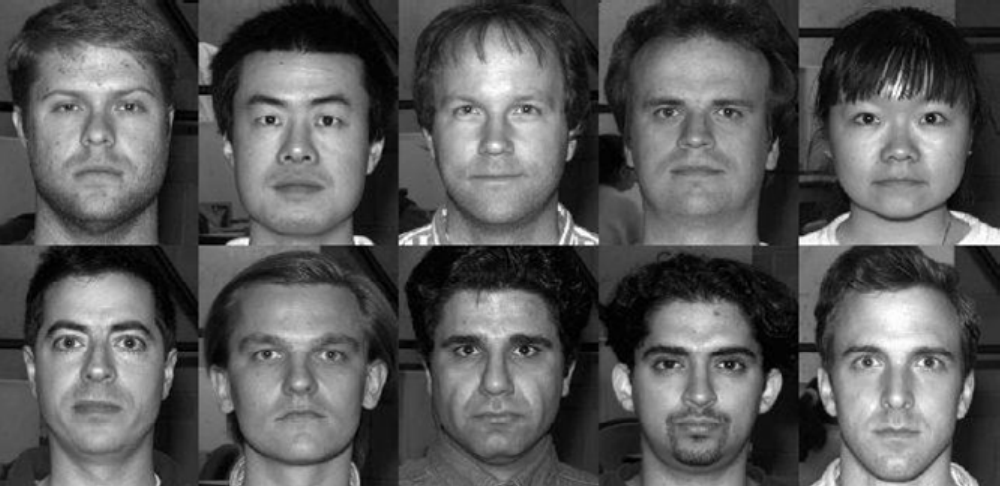}
\caption{Some examples from the Yale Face Database B.}
\label{yaleB}
\end{center}
\end{figure}
%In the experiment, we randomly pick up half of the data for training and the rest for testing.

To compute the Affine-Invariant kernel (shown in Table~\ref{PSDkernel}) on the Riemannian manifold, we use covariance descriptors \cite{Tuzel2006Region} computed from the feature vector $\Big[ x,y,I_{xy},|G_x|,|G_y|, \sqrt{G_x^2+G_y^2}, |G_{xx}|,|G_{yy}|, \arctan\big( \frac{|G_x|}{|G_y|}\big) \Big]$, where $x$, $y$ are pixel locations, $I_{xy}$ is the grayscale value at $xy$-coordinate location,  and $G_x$, $G_y$ are the first order intensity derivatives.
Likewise, $G_{xx}$, $G_{yy}$ are the second order intensity derivatives.
Accordingly, each face image in this database can be represented by the covariance matrix, \emph{i.e.} a 9 $\times$ 9 SPD matrix.
And then, the similarity between two images can be evaluated by the Affine-Invariant kernel.

\begin{table*}
\centering
  \caption{Statistics of the Affine-Invariant(Aff-I) kernel and the test classification accuracy(\%) on Yale Face Database B.}
  \label{tab:ESC}
  \begin{tabular}{ccc|cc|ccccccccccc}
    \toprule
   &$\mu_{\min}$  &$\mu_{\max}$	&SVM(Log-E) & KLR(Log-E) &Flip &Clip &Shift &TDCASVM &KSVM &CCICP-GD &CCICP-SGD\\
  \midrule
  &-160.28
  &670.05
  &97.6$\pm$0.5
  &97.3$\pm$0.3
  &97.4$\pm$0.2 &95.6$\pm$2.9 &96.0$\pm$1.6 &97.7$\pm$0.5 &97.8$\pm$0.3
  &98.3$\pm$0.8 &97.8$\pm$1.7\\
  \bottomrule
\end{tabular}
\end{table*}

Table~\ref{tab:ESC} presents statistics including the minimum and maximum eigenvalues of the Affine-Invariant kernel, i.e. $\mu_{\min}$ and $\mu_{\max}$.
Note that such kernel on this data shows highly indefinite.
We compare the proposed CCICP-GD and CCICP-SGD algorithms with five indefinite learning based methods with the Affine-Invariant kernel including ``Flip", ``Clip", ``Shift", ``TDCASVM" and ``KSVM".
And also, the Log-Euclidean kernel \cite{Jayasumana2013Kernel}, a PD kernel, is incorporated into two representative classifiers SVM and KLR for comparisons.
Table~\ref{tab:ESC} reports the average test accuracy(\%) across abovementioned algorithms.
One can see that these kernel approximation based methods ``Flip", ``Clip" and ``Shift" do not achieve satisfactory performance when compared with other positive definite/indefinite kernel learning based algorithms. This is because the above methods actually change the indefinite matrix itself.
Among these indefinite learning based algorithms, the proposed CCICP-GD algorithm performs better than ``TDCASVM" and ``KSVM" with a margin of 0.6\% and 0.5\% on the test classification accuracy.
Besides, the proposed CCICP-SGD algorithm achieves a comparable performance among these methods.
The experimental results on this data set reinforce to demonstrate the effectiveness of the proposed algorithms with various indefinite kernels.

\subsection{Effect of the Inexact Parameter $\epsilon$}
As aforementioned, the only difference between CCCP-GD and CCICP-GD is the selection of the inexact parameter $\epsilon$.
When $\epsilon$ approaches to zero, the CCICP-GD algorithm degenerates to a standard CCCP-GD algorithm.
%The inexact parameter $\epsilon$ determines the early termination condition in our IKLR model.
In our experiment, $\epsilon$ is set to 0.0001 in the CCCP-GD algorithm while we choose $\epsilon=1$ in CCICP-GD.
Based on this, this section investigates how its variation (\emph{i.e.} 0.0001, 0.001, 0.01, 0.1, 0.5, 1, 5) in the inexact solving scheme influences the test accuracy and the computational cost during training.
Two data sets appeared in Section \ref{sec:uci} are used here for our experiments, namely a small-scale data set  \emph{monks1} and a larger one \emph{guide1-t}.

\begin{figure}
\begin{center}
\subfigure[]{\label{ga1flw}
\includegraphics[width=0.23\textwidth]{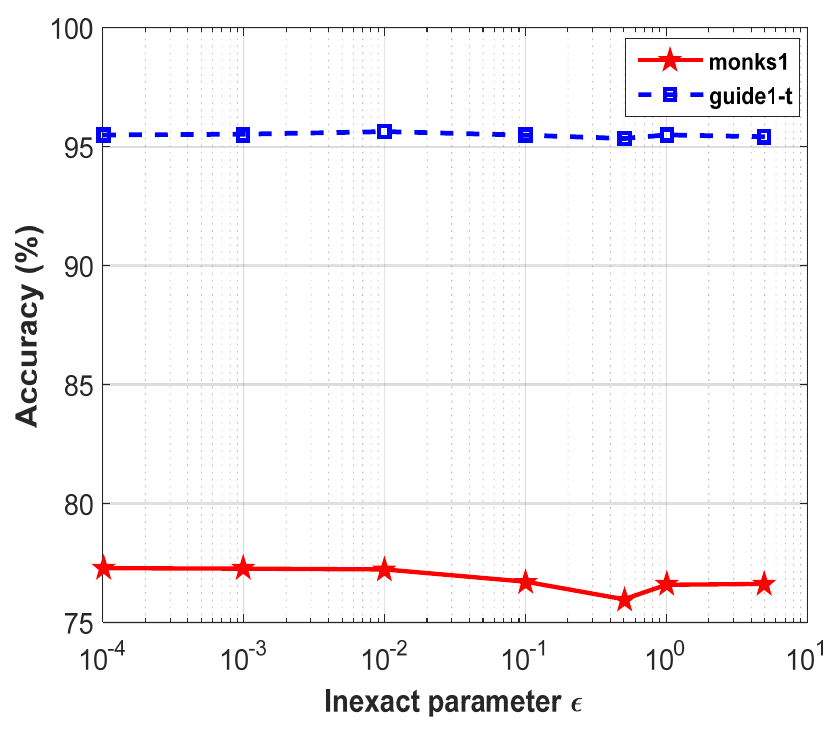}}
\subfigure[]{\label{gacvflw}
\includegraphics[width=0.23\textwidth]{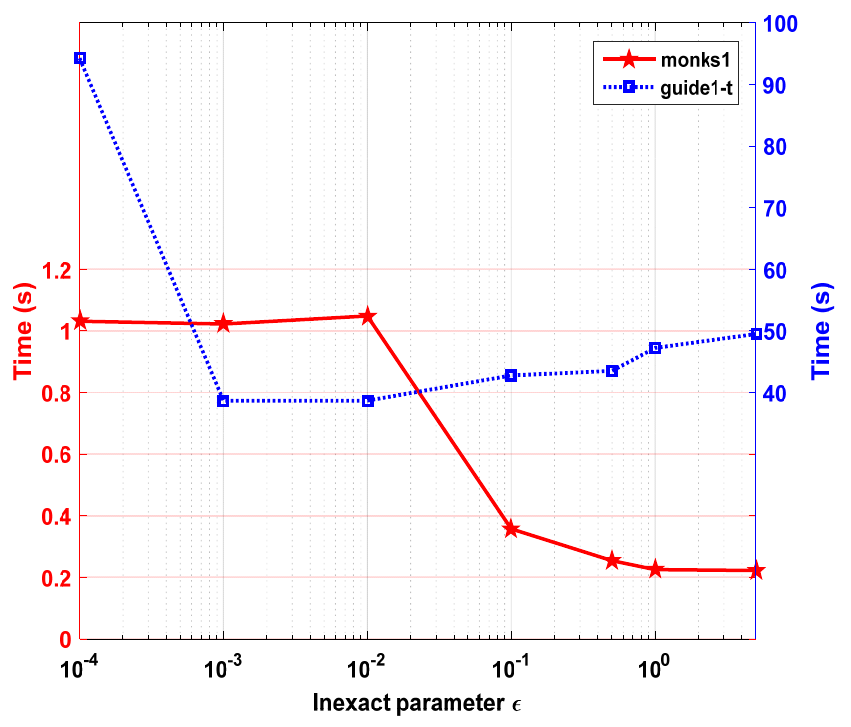}}
\caption{Influence of tuning $\epsilon$ on \emph{monks1} (red) and \emph{guide1-t} (blue), with accuracy (a) and time (b) versus iteration.}
\label{mu}
\end{center}\vspace{-0.cm}
\end{figure}

Fig.~\ref{ga1flw} illustrates that the performance of CCICP-GD is generally not sensitive to $\epsilon$ on such two different types of data sets.
In Fig.~\ref{gacvflw}, on \emph{monks1}, the training time cost does not dramatically decrease when $\epsilon$ ranges from 0.0001 to 0.01, and then it rapidly falls down.
We can conclude that CCICP-GD ($\epsilon=\text{1}$) is much efficient than CCCP-GD ($\epsilon=\text{0.0001}$), and thus such tendency demonstrates the effectiveness of the proposed inexact scheme.
Meanwhile, on \emph{guide1-t}, the setting with $\epsilon=\text{0.001}$ spends the minimum time during training, which almost cuts by half when compared with the situation of initial value $\epsilon=\text{0.0001}$.
After that, the time cost steadily increases, which shows an ``abnormal" tendency on $\epsilon$.
This is because the algorithm with an inexact solution sometimes requires more iterations to converge to a stationary point.
However, CCICP-GD ($\epsilon=\text{1}$) is still efficient than CCCP-GD ($\epsilon=\text{0.0001}$) in this data set.
%From above observations, in terms of the training time cost, our method shows different tendencies on these two data sets.
Generally, CCICP-GD with larger $\epsilon$ often spends less training time than the setting with smaller one to converge.
This is because the termination condition can be significantly relaxed, which has been well demonstrated on these two data sets.
%For example, in \emph{guide-t} dataset, when $\epsilon$ exceeds 0.001, the training time does not reduce as we expected.

\subsection{Algorithm Convergence}
\label{sec:conv}
Fig.~\ref{conv} shows the convergence of IKLR with three optimization algorithms on \emph{monks1} and \emph{ijcnn1-tr}.
\begin{figure}
\begin{center}
\subfigure[]{\label{conv1}
\includegraphics[width=0.23\textwidth]{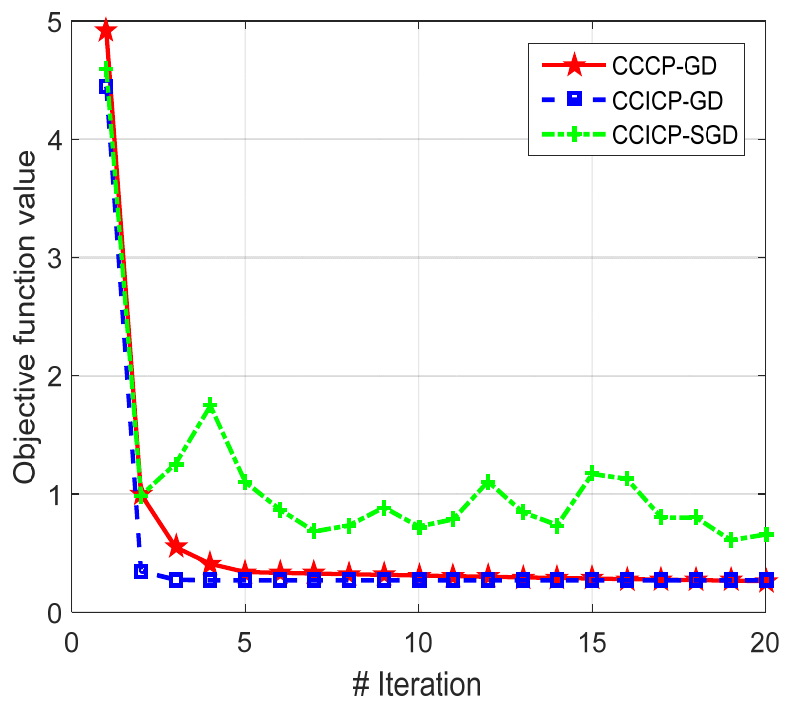}}
\subfigure[]{\label{conv2}
\includegraphics[width=0.23\textwidth]{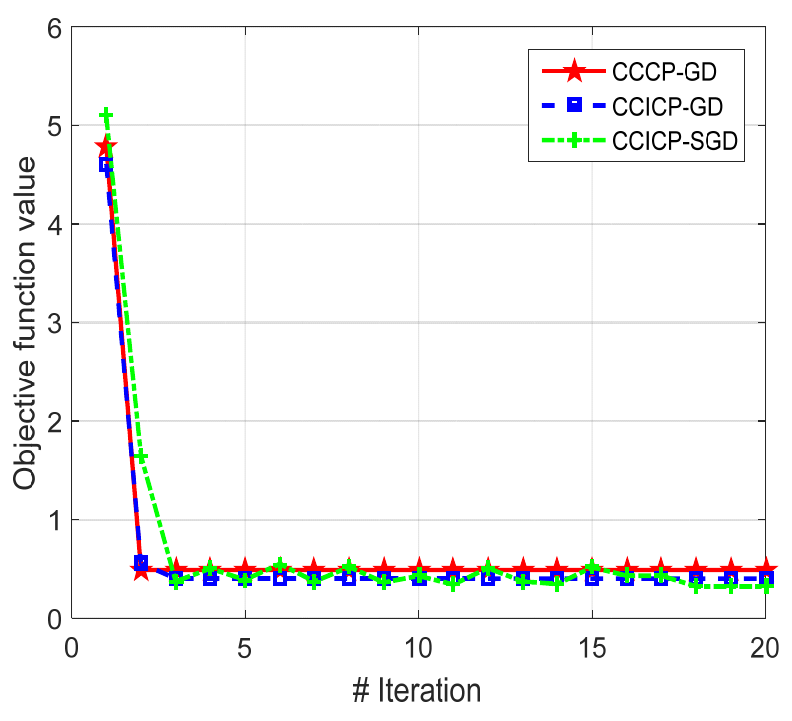}}
\caption{Convergence plots for CCCP-GD (red), CCICP-GD (blue) and CCICP-SGD (green) on \emph{monks1} (a) and \emph{ijcnn1-tr} (b), respectively.}
\label{conv}
\end{center}\vspace{-0.cm}
\end{figure}
It can be observed that, in Fig.~\ref{conv1}, on the \emph{monks1} data set, CCICP-GD converges within 5 iterations, but CCCP-GD takes 16 iterations to converge.
In Fig.~\ref{conv2}, both CCCP-GD and CCICP-GD converge fast on the \emph{ijcnn1-tr} data set.
%As a consequence,  compared with the conventional CCCP algorithm, the proposed CCICP generating an inexact solution indeed requires more iterations to converge a stationary point.
%Such operation might incur more computation load for CCICP-GD to achieve convergence.
%Accordingly, CCICP-GD is less efficient than CCCP-GD on this data set as demonstrated by Table.~\ref{tabptb}.
The above two gradient-based algorithms (CCCP-GD and CCICP-GD) monotonically decrease in each iteration.
However, in our CCICP-SGD version, it cannot be guaranteed to monotonically decrease due to its random scheme.
Instead, it just converges to a stationary point in expectation as shown in Fig.~\ref{conv}.

\subsection{Different Random Initializations}
\begin{table}[tp]
  \centering
  \fontsize{6.5}{8}\selectfont
  \begin{threeparttable}
  \caption{CCICP-GD and CCICP-SGD with different random initializations on the \emph{monks1} and \emph{guide1-t} dataset.}
  \label{tab:init}
    \begin{tabular}{ccccccccccc}
    \toprule
    \multirow{1}{*}{Method}&
    \multicolumn{2}{c}{CCICP-GD} &\multicolumn{2}{c}{CCICP-SGD} \cr
     \cmidrule(lr){2-3} \cmidrule(lr){4-5}
    {Initialization}&\emph{monks1}&\emph{guide1-t}&\emph{monks1} &\emph{guide1-t}\cr
    \midrule
        ${\bm \alpha}^{(0)}=\bm{0}$&0.694 &0.925 &0.662$\pm$0.050 &0.926$\pm$0.003 \cr
    ${\bm \alpha}^{(0)}=\bm{1}$ &0.697 &0.925 &0.675$\pm$0.037 & 0.919$\pm$0.016\cr
    ${\bm \alpha}^{(0)}=-\bm{1}$ &0.694 &0.925 &0.661$\pm$0.043 &0.923 $\pm$0.008 \cr
    $\alpha^{(0)}_i \in (0,1)$ &0.706 &0.927 &0.663$\pm$0.054 &0.926$\pm$0.049 \cr
    \bottomrule
    \end{tabular}
    \end{threeparttable}
\end{table}

The proposed algorithms, CCICP-GD and CCICP-SGD, have been experimentally demonstrated to be converge as illustrated in Section~\ref{sec:conv}.
Since the IKLR model is non-convex, different initializations might lead to different stationary points.
Here we choose two data sets, \emph{monks1} and \emph{guide1-t}, to investigate the influence of our algorithms with different initializations on the final classification accuracy.
Such two data sets are conducted with 10 runs on a fixed (or pre-defined) training and test data for fair comparisons.
As suggested in \cite{Dauphin2014Identifying}, in our experiment, we choose four different initializations with small values, \emph{i.e.} ${\bm \alpha}^{(0)}=\bm{0}$, ${\bm \alpha}^{(0)}=\bm{1}$, ${\bm \alpha}^{(0)}=-\bm{1}$, and the randomly initialization $\alpha^{(0)}_i \in (0,1)$.
By doing so, such small initialization values can guarantee that the objective function value in Eq.~\eqref{iklrmain} is always positive during the optimization process.
Table~\ref{tab:init} demonstrates that different initializations near zero often lead to slight fluctuation on the final classification accuracy.

\subsection{Discussion on CCICP-SGD}
\label{sec:diss}
As aforementioned in Section \ref{sec:compare}, the inexact parameter $\epsilon$ in CCICP-SGD is fixed to 0.0001 because SGD achieves the similar purpose with the inexact scheme \emph{i.e.}, $\epsilon=\text{1}$.
However, in Table~\ref{tabptb}, it can be observed that CCICP-GD is much efficient than CCICP-SGD.
Such time cost reduction motivates us to see how fast our algorithm can be when SGD comes to the inexact scheme.
Accordingly, we investigate the performance of CCICP-SGD with the early termination condition (\emph{i.e.} $\epsilon=1$), termed as ``CCICP-SGD-I".

\begin{table*}[tp]
  \centering
  \fontsize{6.5}{8}\selectfont
  \begin{threeparttable}
  \caption{Results of CCICP-SGD and CCICP-SGD-I on four larger than small scale data sets.}
  \label{sgdi}
    \begin{tabular}{ccccccccccc}
    \toprule
    \multirow{1}{*}{Dataset}&
    \multicolumn{2}{c}{EEG} &\multicolumn{2}{c}{guide1-t} &\multicolumn{2}{c}{ijcnn1-tr} &\multicolumn{2}{c}{madelon} \cr
    \cmidrule(lr){2-3} \cmidrule(lr){4-5} \cmidrule(lr){6-7}  \cmidrule(lr){8-9}
    {Method}&CCICP-SGD&CCICP-SGD-I&CCICP-SGD&CCICP-SGD-I &CCICP-SGD&CCICP-SGD-I
    &CCICP-SGD&CCICP-SGD-I\cr
    \midrule
        Accuracy &0.630$\pm$0.042 &0.590$\pm$0.052
&0.947$\pm$0.022  &0.876$\pm$0.037
&0.914$\pm$0.006 &0.912$\pm$0.003
&0.601$\pm$0.028 &0.550$\pm$0.065 \cr
    Training time &8776.4 &71.2754 &714.72 &1.4146 &28.23 &2.445 &160.71 &2.543\cr
    Test time &0.1188 &0.0235 &0.0021 & 0.0023 &0.2258 &0.2378 &0.0011 &0.0025 \cr
    \bottomrule
    \end{tabular}
    \end{threeparttable}
\end{table*}

Table~\ref{sgdi} reports the classification accuracy and the computation cost of CCICP-SGD and CCICP-SGD-I on four larger than small scale data sets.
One can see that CCICP-SGD-I degrades the test accuracy to some extent when compared with CCICP-SGD on \emph{EEG}, \emph{guide1-t}, and \emph{madelon}.
However, CCICP-SGD-I equipped with the inexact scheme extremely accelerates the training process, of which the training time is about one-hundreds or less than that of CCICP-SGD.
On \emph{ijcnn1-tr}, CCICP-SGD-I is more efficient than CCICP-SGD without too much degeneracy on the classification accuracy.
%On \emph{ijcnn1-tr}, CCICP-SGD-I performs better than CCICP-SGD in the terms of classification accuracy, but the inexact scheme does not speed up the training process.
%The underlying reason might be similar with the case between CCCP-GD and CCICP-GD as demonstrated in Section \ref{sec:conv}.
\section{Conclusion}
\label{sec:conclusion}
In this paper, we investigate kernel logistic regression with indefinite kernels in theoretical and algorithmic aspects.
The derived IKLR model is non-convex and further analysed in RKKS with explicit demonstration due to the non-positive definite kernels.
Such non-convex problem can be effectively and efficiently solved by the proposed CCICP equipped with two approximation schemes.
%The introduced inexact schemes including the early termination condition and a stochastic version in CCCP is able to accelerate the training process.
Its GD version using an early stop scheme is able to make the training process efficient;
the stochastic variant of CCICP also has the capability of accelerating the solving process.
The convergence analyses of CCICP-GD and CCICP-SGD are conducted with theoretical guarantees and experimental validation.
%Specifically, the CCICP exhibits quasi-Newton behavior or typically superlinear convergence because the convex part in our IKLR model dominates the concave part.
The classification accuracy of the proposed IKLR model on several benchmarks demonstrates its effectiveness when compared to other positive definite/indefinite kernel learning methods.
%Extensive comparative experiments from multi-modal datasets validate the superiority of the proposed IKLR model to other algorithms with positive definite/indefinite kernels.
%Further, the results also enlighten us to design a proper indefinite kernel and it does not limit to a positive definite kernel.

\section*{Acknowledgements}
The authors would like to thank Jiaxuan Xie from Shanghai Jiao Tong University for his work on the implementation of TDCASVM, and also sincerely appreciate the anonymous reviewers for their insightful comments.
%\setcounter{equation}{\value{mytempeqncnt}}
%\hrulefill
%\vspace*{4pt}
%\end{figure*}
\bibliographystyle{IEEEbib}
\bibliography{refs}

\end{document}